\documentclass[lettersize,journal]{IEEEtran}
\usepackage{graphicx}
\usepackage{cite}
\usepackage{picinpar}
\usepackage{amsmath}
\usepackage{url}
\usepackage{flushend}
\usepackage[latin1]{inputenc}
\usepackage{colortbl}
\usepackage{soul}
\usepackage{pifont}
\usepackage{color}
\usepackage{alltt}
\usepackage{hyperref}
\usepackage{enumerate}
\usepackage{siunitx}
\usepackage{breakurl}
\usepackage{epstopdf}
\usepackage{pbox}
\usepackage{bm}
\usepackage{array} 
\usepackage{svg}
\usepackage{float}
\usepackage{booktabs}
\usepackage{amssymb}
\usepackage{xurl}
\usepackage{xcolor}     

\usepackage[caption=false,font=scriptsize,labelfont=sf,textfont=sf]{subfig}

\usepackage{makecell}
\usepackage{multirow,tabularx}
\usepackage{algorithm}
\usepackage{algorithmic}
\usepackage{threeparttable}
\usepackage[most]{tcolorbox}

\usepackage{mathtools} 

\newcounter{proofcounter} 
\newtheorem{assumption}{Assumption}
\newtheorem{theorem}{Theorem}
\newtheorem{lemma}{Lemma}
\newenvironment{proof}{{\indent \indent \it Proof~\theproofcounter:}\stepcounter{proofcounter}}{\hfill $\blacksquare$\par}

\definecolor{myyellow}{RGB}{240,230,140}
  
\newtheorem{remark}{Remark}

\newtcbox{\cbox}[1][]{
    on line,
    boxrule=0pt,        
    colback=#1,         
    colframe=#1,
    arc=2pt,            
    boxsep=0.1pt,         
    left=1pt,
    right=1pt,
    top=1pt,
    bottom=1pt,
}

\begin{document}

\title{Autonomous UAV Pipeline Near-proximity Inspection via Disturbance-Aware Predictive Visual Servoing}

\author{
\vskip 1em
	Wen Li$^{1}$, Hui Wang$^{1}$, Jinya Su$^{1}$, Cunjia Liu$^{2}$, Wen-Hua Chen$^{3}$,~\emph{IEEE Fellow} and Shihua Li$^{1}$,~\emph{IEEE Fellow}
\thanks{*This work was supported by the National Natural Science Foundation of China under Grant 62303110.
Corresponding author: Prof. Jinya Su.}
\thanks{$^{1}$Wen Li, Hui Wang, Jinya Su, and Shihua Li are with the School of Automation, Key Laboratory of Measurement and Control of CSE, Ministry of Education, 
Southeast University, Nanjing 210096, China. {\tt\small lwen@seu.edu.cn;
220252044@seu.edu.cn; sucas@seu.edu.cn; lsh@seu.edu.cn}.}%
\thanks{$^{2}$Cunjia Liu is with the Department of Aeronautical and Automotive Engineering, Loughborough University, LE11 3TU, United Kingdom. {\tt\small c.liu5@lboro.ac.uk}.}%
\thanks{$^{3}$Wen-Hua Chen  is with the Research Centre for Low Altitude Economy and the  Department of Aeronautical and Aviation Engineering, The Hong Kong Polytechnic University, Hong Kong, China. {\tt\small wenhua.chen@polyu.edu.hk}.}%
}

\maketitle
\begin{abstract}

Reliable pipeline inspection is critical to safe energy transportation, but is constrained by long distances, complex terrain, and risks to human inspectors. Unmanned aerial vehicles provide a flexible sensing platform, yet reliable autonomous inspection remains challenging. This paper presents an autonomous quadrotor near-proximity pipeline inspection framework for three-dimensional scenarios based on image-based visual servoing model predictive control (VMPC). A unified predictive model couples quadrotor dynamics with image feature kinematics, enabling direct image-space prediction within the control loop. To address low-rate visual updates, measurement noise, and environmental uncertainties, an extended-state Kalman filtering scheme with image feature prediction (ESKF-PRE) is developed, and the estimated lumped disturbances are incorporated into the VMPC prediction model, yielding the ESKF-PRE-VMPC framework. A terrain-adaptive velocity design is introduced to maintain the desired cruising speed while generating vertical velocity references over unknown terrain slopes without prior terrain information. The framework is validated in high-fidelity Gazebo simulations and real-world experiments. In real-world tests, the proposed method reduces RMSE by \(52.63\%\) and \(75.04\%\) in pipeline orientation and lateral deviation in the image, respectively, for straight-pipeline inspection without wind, and successfully completes both wind-disturbance and bend-pipeline tasks where baseline method fails. An open-source nano quadrotor is modified for indoor experimentation\footnote{The hardware and code are open-sourced at \url{https://github.com/lw-seu/Crazyflie-modification} and \url{https://github.com/lw-seu/VS_mpc} (to be released). The project webpage is available at \url{https://eskf-pre-vmpc.github.io/}.}. 

\end{abstract}

\renewcommand{\thefootnote}{}
\footnotetext{\textbf{Copyright:} This work has been submitted to the IEEE for possible publication. Copyright may be transferred without notice, after which this version may no longer be accessible.}





\section{INTRODUCTION}

Pipelines are indispensable infrastructures for modern energy transportation, and maintaining their structural integrity is critical. Because these networks are widely distributed and exposed to ageing, corrosion, and harsh operating conditions, periodic inspection is indispensable. Defects such as corrosion, deformation, or cracks may cause leakage, service interruption, severe environmental contamination, and even catastrophic fire or explosion events~\cite{iqbal2017inspection}. 
In practice, however, pipeline inspection remains difficult because many pipelines extend across remote areas, uneven terrain, and safety-critical environments where routine manual patrols are inefficient, costly, and potentially dangerous. Unmanned Aerial Vehicles (UAVs), with their mobility, flexible sensing capability, and ability to access 
hard-to-reach areas, have consequently emerged as an attractive alternative to conventional inspection approaches, offering the potential to improve efficiency while reducing operational risks and manpower requirements~\cite{su2023ai}.

Despite this promise, many UAV inspection systems still rely heavily on manual piloting or operator supervision~\cite{liu2022uav}. Achieving a higher level of autonomy requires a control strategy that remains accurate and reliable in close-proximity inspection scenarios. Although GPS signals are commonly used for aerial navigation, they can become unreliable in GPS-denied or degraded environments~\cite{gyagenda2022review} and generally do not provide the precision needed for low-altitude, near-structure 
inspection. Visual servoing provides a compelling alternative by directly regulating the UAV motion using onboard visual feedback~\cite{chaumette2006visual}. Among different visual servoing paradigms~\cite{xu2025matters}, Image-Based Visual Servoing (IBVS) is attractive for inspection tasks because it controls the target directly in the image plane, avoiding explicit three-dimensional (3D) reconstruction and thereby reducing sensitivity to calibration inaccuracies~\cite{zhang2021robust}. 

Nevertheless, translating IBVS into a reliable control solution for quadrotor inspection is far from straightforward. Small quadrotors
are characterized by low inertia, limited payload, underactuated dynamics, and uncertainties caused by unmodeled aerodynamics and parameter variations~\cite{eschmann2025raptor}. These properties make non-model-based IBVS schemes difficult to apply, 
since they typically neglect the vehicle dynamics as well as visibility-related constraints. As a result, aggressive or delayed corrective actions may drive the visual target out of FOV, leading to degraded servoing performance or even loss of closed-loop stability~\cite{yang2024robust}. This issue becomes even more critical in practice, where visual measurements are often available at a lower update rate and with poorer quality than the onboard state feedback. Model Predictive Control (MPC) provides a principled way to address these challenges by explicitly incorporating system dynamics and constraints into a finite-horizon optimization framework. By optimizing the future control sequence online, MPC can improve both motion quality and visibility preservation, thereby offering a more reliable basis for autonomous UAV inspection~\cite{zhang2021robust}.

However, the performance of MPC can deteriorate significantly when the predictive model is affected by uncertainties and external disturbances~\cite{zhang2025damppi}, both of which are unavoidable in outdoor or near-structure pipeline inspection. To improve robustness, a broad range of MPC formulations has been developed, including robust MPC~\cite{zhang2021robust}, Tube-MPC~\cite{langson2004robust}, and disturbance estimation-based MPC approaches~\cite{liu2012tracking}. Among them, disturbance estimation-based MPC is particularly attractive because it can be interpreted as a two-degree-of-freedom control architecture: the nominal predictive controller retains desirable tracking behaviour, while the disturbance estimation channel actively compensates for model mismatch and external perturbations~\cite{chen2015disturbance,sun2017disturbance}. Such a structure is well suited to UAV inspection tasks, where both responsiveness and robustness are simultaneously required.

Existing disturbance estimation-based MPC methods can generally be grouped into three categories: (1) direct feedforward compensation using estimated disturbances, (2) prediction model augmentation with constant disturbance estimates, and (3) incorporation of disturbance preview into the predictive model. For instance, \cite{liu2012tracking} proposed an explicit MPC framework with a nonlinear disturbance observer to provide feedforward compensation for small-scale helicopters under wind disturbances. 
\cite{yang2015design,zhang2025damppi} enhanced prediction accuracy by augmenting the MPC model with disturbance estimates assumed constant over the prediction horizon. More recently, disturbance preview has been introduced into predictive control formulations~\cite{grasshoff2019model,zhan2022computationally,zhang2026predmppi}, where the preview information may be obtained from physical disturbance models~\cite{walker2025nonlinear}, parametric modelling and estimation~\cite{castillo2021predicting}, or non-parametric learning methods
~\cite{grasshoff2019model}. 

Despite these advances, 
there is still a lack of \textit{a unified VMPC framework that simultaneously addresses quadrotor dynamics, low-rate visual updates, measurement noise, and external disturbances, with validation in realistic three-dimensional pipeline inspection scenarios}. Motivated by the above limitations, this paper develops a unified framework for autonomous quadrotor pipeline inspection in 3D environments. A comparative summary of representative visual servoing and UAV inspection approaches is provided in Table~\ref{comparison}, and the main contributions are summarized as follows:

\begin{enumerate}

\item[(1)] \textbf{Visual servoing modelling}: A unified predictive visual servoing model is established by explicitly coupling quadrotor dynamics with image feature kinematics, enabling integrated model predictive visual servoing for UAV pipeline inspection.

\item[(2)] \textbf{ESKF-PRE-VMPC}: An ESKF-PRE-VMPC framework is developed by integrating an extended-state Kalman filter with image feature prediction into visual servoing MPC, thereby improving robustness against measurement noise, low-rate vision updates, and disturbances.

\item[(3)] \textbf{Terrain-adaptive velocity control}: Unlike 
planar or quasi-2D inspection scenarios~\cite{xing2023autonomous, velasco2024visual}, an intermediate MPC state design is introduced to maintain the desired cruising speed while adapting the vertical motion to complex 3D terrains without prior terrain knowledge.

\item[(4)] \textbf{Comparative verification}: The proposed framework is validated through high-fidelity Gazebo simulations and real-world quadrotor experiments in 3D pipeline inspection scenarios, demonstrating improved stability and tracking robustness over three representative baselines.

\end{enumerate}

\begingroup
\setlength{\tabcolsep}{1.8pt}
\renewcommand{\arraystretch}{1.08}
\begin{table}[t]
\centering
\scriptsize
\caption{Comparison of visual servoing MPC variants. \emph{Img.}=image, \emph{Dyn.}=dynamics, \emph{DC}=disturbance compensation, \emph{Inspect. App.}=inspection application, \emph{Feat. Pred.}=feature prediction.}
\label{comparison}
\vspace{-1.5mm}





\begin{tabularx}{\linewidth}{@{}p{0.39\linewidth}
    >{\centering\arraybackslash}X
    >{\centering\arraybackslash}X
    >{\centering\arraybackslash}X
    >{\centering\arraybackslash}X@{}}
\toprule
\textbf{Methods} &
\makecell[c]{\textbf{Img.}\\\textbf{DC}} &
\makecell[c]{\textbf{Dyn.}\\\textbf{DC}} &
\makecell[c]{\textbf{Inspect.}\\\textbf{App.}} &
\makecell[c]{\textbf{Feat.}\\\textbf{Pred.}} \\
\midrule

\makecell[l]{Cascaded VMPC\cite{roque2020fast,velasco2024visual}} &
\ding{53} &
\ding{53} &
\makecell[c]{\checkmark\\\scriptsize\cite{velasco2024visual}} &
\ding{53} \\

\makecell[l]{Integrated VMPC-K\cite{yang2024robust,qiu2020disturbance}} &
\makecell[c]{\checkmark\\\scriptsize\cite{qiu2020disturbance}} &
\ding{53} &
\ding{53} &
\ding{53} \\

\makecell[l]{Integrated VMPC-D\cite{zhang2021robust,xing2023autonomous,conference}} &
\makecell[c]{\checkmark\\\scriptsize\cite{conference}} &
\makecell[c]{\checkmark\\\scriptsize\cite{conference}} &
\makecell[c]{\checkmark\\\scriptsize\cite{xing2023autonomous,conference}} &
\ding{53} \\

\makecell[l]{\textcolor{blue}{\textbf{ESKF-PRE-VMPC}}} &
\textcolor{blue}{\textbf{\checkmark}} &
\textcolor{blue}{\textbf{\checkmark}} &
\textcolor{blue}{\textbf{\checkmark}} &
\textcolor{blue}{\textbf{\checkmark}} \\
\midrule

\multicolumn{5}{>{\columncolor{gray!10}}p{\dimexpr\columnwidth-2\tabcolsep\relax}@{}}%
{\textbf{Open gap:} \textit{A unified VMPC framework that addresses low-rate vision updates, measurement noise, and disturbances, with validation in realistic 3D pipeline inspection scenarios, remains limited in the literature.}} \\
\bottomrule
\end{tabularx}

\vspace{-3.5mm}
\end{table}
\endgroup

This article is an extended version of our conference paper~\cite{conference}, where the main extensions
are summarized below:
\begin{enumerate}
    \item[(1)] \textbf{Enhanced estimation and visual prediction:} In addition to lumped disturbance estimation, this work explicitly addresses measurement noise and low-rate visual updates by developing an enhanced estimation scheme with image feature prediction, thereby providing more reliable high-frequency visual information for closed-loop control;
    
    \item[(2)] \textbf{Closed-loop analysis:} A theoretical analysis is introduced to establish the recursive feasibility and asymptotic stability properties of the proposed control framework;
    
    
    \item[(3)] \textbf{Fast real-world validation:} Beyond more realistic simulation  verification, real-world quadrotor experiments are conducted to further evaluate the proposed approach. To facilitate rapid indoor validation,
    a modified Crazyflie~\cite{giernacki2017crazyflie} platform with enhanced payload capacity is developed and openly published.
\end{enumerate}

\textit{\textbf{Acronyms and explanations}}: \textbf{MPC}: \textbf{M}odel \textbf{P}redictive \textbf{C}ontrol; \textbf{VMPC}: \textbf{V}isual servoing \textbf{MPC}; \textbf{Cascaded VMPC}: Visual servoing module generates a reference signal, which is then tracked by an MPC; \textbf{Integrated VMPC-K}: Image feature kinematics are used as the predictive model in MPC, which directly computes the control inputs; \textbf{Integrated VMPC-D}: Both the image features kinematics and the quadrotor dynamics are incorporated into the MPC formulation. \textbf{ESKF-PRE-VMPC}: \textbf{E}xtended \textbf{S}tate based \textbf{K}alman \textbf{F}ilter with Image Features \textbf{Pre}diction (ESKF-PRE) \textbf{VMPC} framework.

\section{PRELIMINARIES}\label{sec:preliminaries}

\subsection{Coordinate Frames and Notation}

Five coordinate frames are employed in this work, as illustrated in Fig.~\ref{frame}: the inertial world frame \(W\), the quadrotor body frame \(B\), the camera frame \(C\), and two 2D frames on the image plane, namely the pixel image frame \(I\) and the normalized image frame \(\Gamma\). These frames are represented by the orthonormal bases \(\{\vec{x}_W,\vec{y}_W,\vec{z}_W\}\), \(\{\vec{x}_B,\vec{y}_B,\vec{z}_B\}\), \(\{\vec{x}_C,\vec{y}_C,\vec{z}_C\}\), \(\{\vec{u},\vec{v}\}\), and \(\{\vec{x}_{\Gamma},\vec{y}_{\Gamma}\}\), respectively. The axis \(\vec{z}_W\) is aligned with the gravity direction. The body frame \(B\) is rigidly attached to the quadrotor, while the camera frame \(C\) is centered at the optical center of the onboard camera.
\begin{figure}[htbp!]
    \centering
    \includegraphics[width=0.9\linewidth, height=0.5\linewidth]{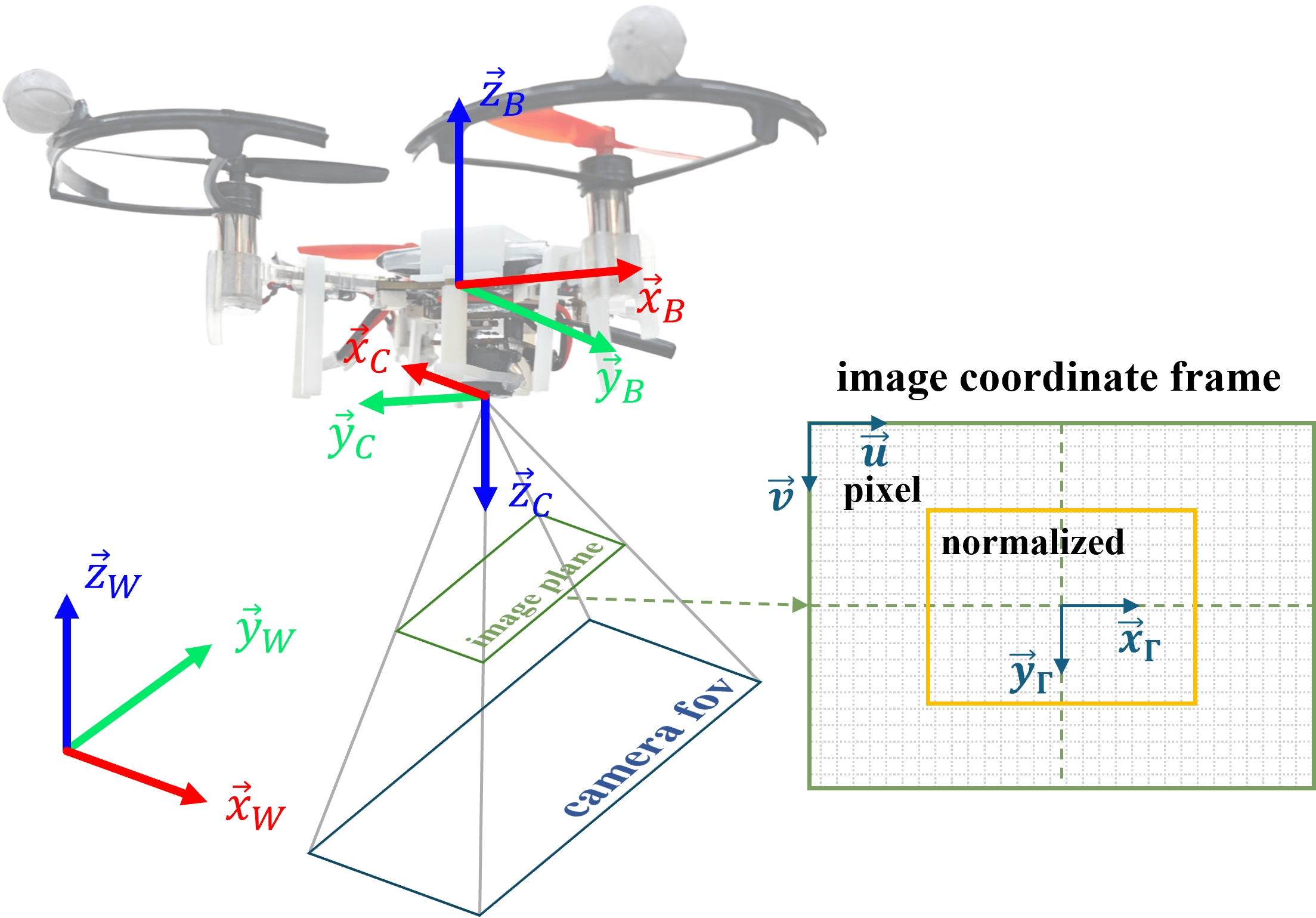}
    \caption{\textbf{Coordinate frames} in quadrotor inspection.}
    \label{frame}
    \vspace{-0.2cm}
\end{figure}


In this study, bold lowercase letters denote vectors, bold uppercase letters denote matrices, and regular letters denote scalars. For frame-dependent variables, a presuperscript indicates that the quantity is expressed in the corresponding local frame, whereas a postsuperscript indicates that the quantity is expressed in the world frame. For example, \(^{B}\bm{\omega}\) denotes the angular velocity of the body frame relative to the world frame, expressed in frame \(B\), whereas \(\bm{v}^{B}\) denotes the linear velocity of the body frame relative to the world frame, expressed in frame \(W\).

The quadrotor attitude is parameterized by the Euler-angle vector \(\bm{\psi}^{B}=(\phi,\beta,\alpha)^\mathsf{T}\), where \(\phi\), \(\beta\), and \(\alpha\) represent the roll, pitch, and yaw angles, respectively. The corresponding rotation matrix \(\bm{R}_{WB}\in \mathcal{SO}(3)\) from frame \(B\) to frame \(W\) is
\begin{equation}\nonumber
\bm{R}_{WB}=
\begin{bmatrix}
\mathrm{c}\alpha\,\mathrm{c}\beta &
\mathrm{c}\alpha\,\mathrm{s}\beta\,\mathrm{s}\phi-\mathrm{s}\alpha\,\mathrm{c}\phi &
\mathrm{c}\alpha\,\mathrm{s}\beta\,\mathrm{c}\phi+\mathrm{s}\alpha\,\mathrm{s}\phi \\[3pt]
\mathrm{s}\alpha\,\mathrm{c}\beta &
\mathrm{c}\alpha\,\mathrm{c}\phi+\mathrm{s}\alpha\,\mathrm{s}\beta\,\mathrm{s}\phi &
\mathrm{s}\alpha\,\mathrm{s}\beta\,\mathrm{c}\phi-\mathrm{c}\alpha\,\mathrm{s}\phi \\[3pt]
-\mathrm{s}\beta &
\mathrm{c}\beta\,\mathrm{s}\phi &
\mathrm{c}\beta\,\mathrm{c}\phi
\end{bmatrix},
\end{equation}
where \(\mathrm{c}(\cdot)\) and \(\mathrm{s}(\cdot)\) are 
for \(\cos(\cdot)\) and \(\sin(\cdot)\), respectively.

\subsection{Inspection Task and Multi-Rate Measurement}

The considered task is autonomous visual inspection of 3D pipeline structures using a quadrotor equipped with a downward-facing onboard camera. Instead of relying on absolute positioning, the control objective is formulated directly in the image space: the pipeline should remain close to the image center with the desired orientation and apparent scale, while the quadrotor maintains the prescribed forward inspection speed. 
The inspection objective can be compactly stated as
\begin{equation}
\bm{\chi} \rightarrow \bm{\chi}_d,\qquad v_h \rightarrow v_{h_d},\qquad \rho \rightarrow \rho_d,
\label{eq:inspection_objective}
\end{equation}
where \(\bm{\chi}\in\mathbb{R}^{2}\) denotes image feature vector for visual servoing, \(v_h\) is the horizontal inspection speed, and \(\rho\) denotes the pipeline width in the image. Their desired values will be explicitly defined in the subsequent sections. 

A practical challenge lies in the multi-rate nature of the sensing pipeline. The quadrotor states available from onboard sensing at a relatively high frequency, whereas image features from the camera are obtained more intermittently at a much lower effective rate.
Let \(t_k=kT_s\) denote the high-rate measurement time grid, and let \(\{t_{k_j}^{\chi}\}_{j=0}^{\infty}\) denote the time instants at which reliable image features are available. The measurement process can be described as
\begin{equation}
\begin{aligned}
\bm{y}^{s}_{k}=\bm{x}^{s}_{k}+\bm{n}^{s}_{k},~
\bm{y}^{\chi}_{k}=
\begin{cases}
\bm{\chi}_k+\bm{n}^{\chi}_k, & k\in\{k_j\},\\
\varnothing, & \text{otherwise},
\end{cases}
\label{eq:vision_measurement}
\end{aligned}
\end{equation}
where \(\bm{x}^{s}_{k}\) denotes quadrotor states available at the high-rate time grid, while \(\bm{n}^{s}_{k}\) and \(\bm{n}^{\chi}_{k}\) denote the measurement noise. Consequently, directly using raw visual measurements in the controller may introduce delayed or inconsistent feedback. This motivates the prediction-enhanced estimation strategy developed later in this paper, where high-rate quadrotor state information and intermittent image measurements are fused to provide reliable full state estimation for predictive control.

\subsection{Quadrotor Model}
\label{Qdynamics}

The quadrotor is modeled as a rigid body with four rotors. Let \(\bm{v}^{B}=(v_x^{B},v_y^{B},v_z^{B})^\mathsf{T}\) and \(\bm{\psi}^{B}\) denote the linear velocity and attitude of the quadrotor, respectively. Following the standard Newton--Euler formulation~\cite{mahony2012multirotor}, the nonlinear quadrotor model used in this work is given by
\begin{equation}
\label{quadrotor dynamics}
\left\{
\begin{aligned}
\dot{\bm{v}}^{B} &= \bm{R}_{WB}\frac{{}^{B}\bm{c}}{m} + \bm{g}, \\
\dot{\bm{\psi}}^{B} &= \bm{M}\,{}^{B}\bm{\omega},
\end{aligned}
\right.
\end{equation}
where \({}^{B}\bm{c}=(0,0,c)^\mathsf{T}\) is the collective thrust vector along the \(\vec{z}_B\)-axis of the body frame, \(m\) is the quadrotor mass, and \(\bm{g}=(0,0,-g)^\mathsf{T}\) is the gravity vector.
The matrix \(\bm{M}\) is the Jacobian that maps the Euler-angle rates to the body angular velocity, and \({}^{B}\bm{\omega}=({}^{B}\omega_x,{}^{B}\omega_y,{}^{B}\omega_z)^\mathsf{T}\) is the angular velocity expressed in frame \(B\).


\section{MODELLING}

\subsection{Image Kinematics for IBVS}
\label{Ikinematics}

According to the pinhole camera model~\cite{Hartley2004}, a point  \(\bm{p}_C=(x_C,y_C,z_C)\) expressed in the camera frame satisfies
\begin{equation}
\label{pinhole}
\left[
\begin{matrix}
\bm{p}_I^\mathsf{T} \\
1
\end{matrix}\right]
=
\begin{bmatrix}
f_u & 0 & c_u \\
0 & f_v & c_v \\
0 & 0 & 1
\end{bmatrix}
\left[
\begin{matrix}
\bm{p}_{\Gamma}^\mathsf{T} \\
1
\end{matrix}
\right]
=
\bm{K}
\left[
\begin{matrix}
\bm{p}_{\Gamma}^\mathsf{T} \\
1
\end{matrix}
\right],
\end{equation}
where $\bm{K}$ denotes the camera intrinsic matrix, $\bm{p}_{\Gamma}=(x_{\Gamma},y_{\Gamma})=\left(\frac{x_C}{z_C},\frac{y_C}{z_C}\right)$ is the projection of the point onto the normalized image plane, and $\bm{p}_I=(u_I,v_I)$ is the corresponding pixel coordinate in the image frame $I$.

In this work, a straight line $l$, defined by two feature points $\bm{\eta}_1=(x_{\Gamma_1},y_{\Gamma_1})$ and $\bm{\eta}_2=(x_{\Gamma_2},y_{\Gamma_2})$, is used to represent center line of the pipe in the normalized image frame $\Gamma$. These feature points are obtained from detected points in the image frame $I$ via the affine transformation in~\eqref{pinhole}. The image feature vector is defined as \(\bm{\chi}=(\theta,r)^\mathsf{T}\in\mathbb{R}^2\), where \(\theta\) and \(r\) denote the orientation and position of $l$, respectively, i.e.,
\begin{equation}
\label{eq:chi_def}
\bm{\chi}
=
\left(
\begin{matrix}
\theta \\
r
\end{matrix}
\right)
=
\left(
\begin{array}{c}
\arctan\!\left(-\dfrac{y_{\Gamma_2}-y_{\Gamma_1}}{x_{\Gamma_2}-x_{\Gamma_1}}\right) \\[0.8em]
\left(x_{\Gamma_1}-\dfrac{x_{\Gamma_2}-x_{\Gamma_1}}{y_{\Gamma_2}-y_{\Gamma_1}}y_{\Gamma_1}\right)\sin\theta
\end{array}
\right).
\end{equation}
Accordingly, the desired image feature for pipeline centering and vertical alignment is given by
$\bm{\chi}_d=
\left(
\frac{\pi}{2},\,0
\right)^\mathsf{T}.
$
The geometric interpretation of \(\bm{\chi}\) is illustrated in Fig.~\ref{image model}.
\begin{figure}[htbp!]
    \centering
    \includegraphics[width=0.85\linewidth, height=0.6\linewidth]{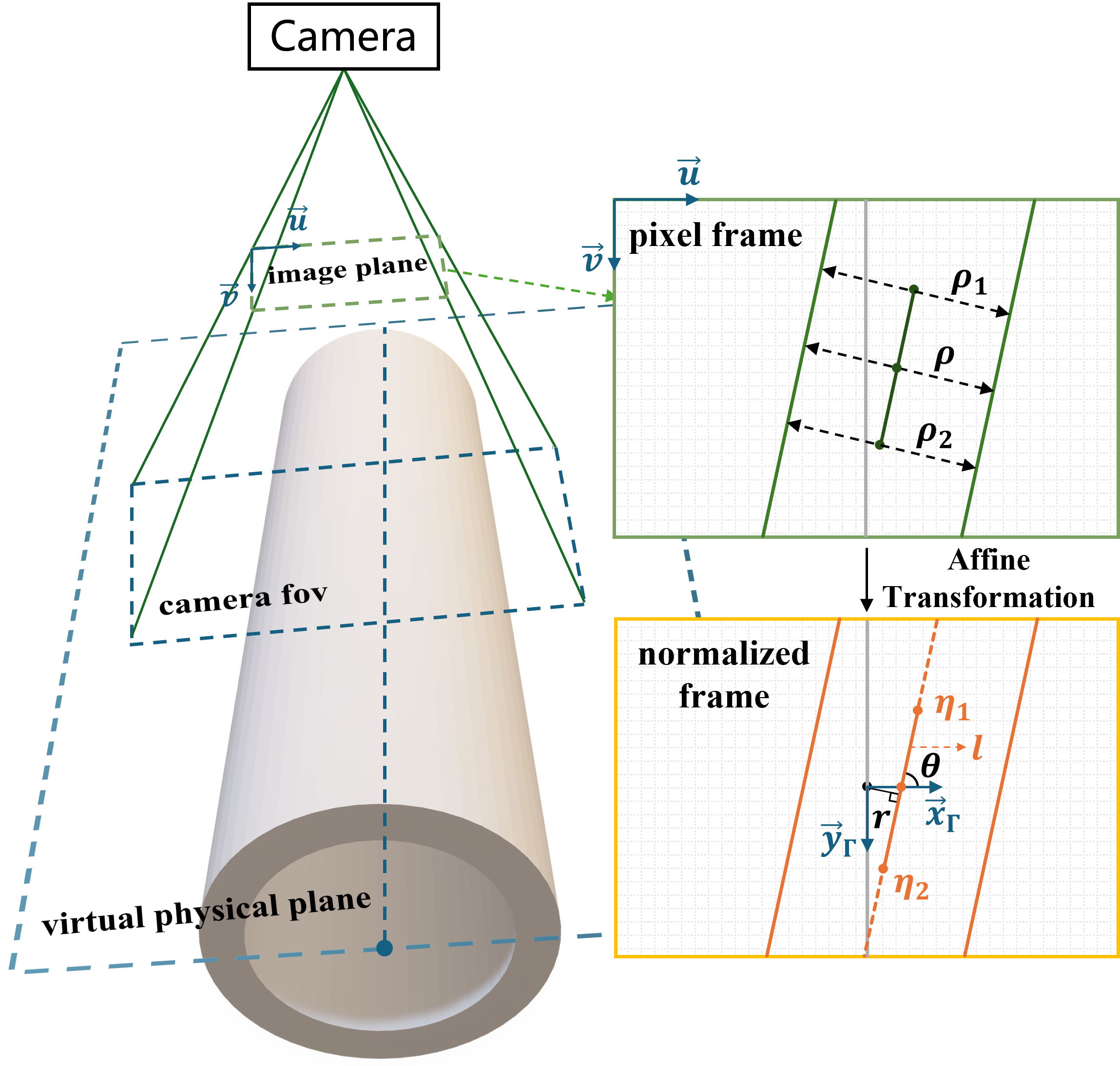}
    \caption{\textbf{Illustration of image features and modelling}, where blue dashed lines indicate the same virtual plane, and \(l\) denotes the corresponding projection on the normalized image plane.}
    \label{image model}
    \vspace{-0.2cm}
\end{figure}

In this work, the pipeline is approximated as a regular cylinder, and its two image-side boundaries are assumed to be reliably extracted by a visual detection module. The line \(l\) can therefore be interpreted as the projection of the center line between the two detected boundaries on a virtual plane parallel to the normalized image plane, as shown in Fig.~\ref{image model}. Based on the known pipeline diameter \(D\) and the observed image width, the depth of a feature point can be estimated from the apparent scale. 
The depth of \(\bm{\eta}_1\) is given by
\begin{equation}
\label{eq:depth_est}
z_{\eta_1}=f_u\cdot\frac{D}{\rho_1},
\end{equation}
where \(z_{\eta_1}\) denotes the estimated depth of \(\bm{\eta}_1\), \(D\) is the actual pipe diameter, and \(\rho_1\) denotes the detected pipeline width in the image frame at the point corresponding to $\bm{\eta}_1$. The depth of \(\bm{\eta}_2\) is obtained analogously.

Following~\cite{velasco2024visual}, the Jacobian matrix \(\bm{J}^{\bm{\chi}}_{{}^{C}\bm{\nu}}\in\mathbb{R}^{2\times 6}\), which relates the image feature \(\bm{\chi}\) to the camera velocity \({}^{C}\bm{\nu}\), is derived using the chain rule as
\begin{equation}
\label{eq:chi_cam_jacobian}
\bm{J}^{\bm{\chi}}_{{}^{C}\bm{\nu}}
=
\bm{J}^{\bm{\chi}}_{\bm{\eta}}
\bm{J}^{\bm{\eta}}_{{}^{C}\bm{\nu}},
\end{equation}
where \(\bm{\eta}=[\bm{\eta}_1^\mathsf{T};\bm{\eta}_2^\mathsf{T}]\in\mathbb{R}^{4}\), \(\bm{J}^{\bm{\chi}}_{\bm{\eta}}\in\mathbb{R}^{2\times4}\) relates \(\bm{\chi}\) to \(\bm{\eta}\), and \(\bm{J}^{\bm{\eta}}_{{}^{C}\bm{\nu}}\in\mathbb{R}^{4\times6}\) is the stacked interaction matrix of the two feature points~\cite{chaumette2006visual}. To connect the image feature dynamics to the vehicle motion, define
$\bm{\nu}=\left[\bm{v}^{B};{}^{B}\bm{\omega}\right]\in\mathbb{R}^{6},$ then the image kinematics can be written as
\begin{equation}
\label{image kinematics}
\dot{\bm{\chi}}=\bm{J}\bm{\nu},
\end{equation}
where
$\bm{J}=\bm{J}^{\bm{\chi}}_{{}^{C}\bm{\nu}}\bm{T}^{-1},$ and \(\bm{T}\in\mathbb{R}^{6\times6}\) is the transformation matrix between the body velocity and camera velocity,
\begin{equation}
\label{eq:T_transform}
\bm{T}=
\left[
\begin{array}{cc}
\bm{R}_{WB}\bm{R}_{BC} & [\bm{t}_{BC}]_{\times}\bm{R}_{BC} \\
\bm{0}_{3\times3} & \bm{R}_{BC}
\end{array}
\right],
\end{equation}
where \(\bm{R}_{BC}\in\mathcal{SO}(3)\) and \(\bm{t}_{BC}\in\mathbb{R}^{3}\) denote the orientation and position of the camera frame \(C\) with respect to the body frame \(B\), respectively, and \([\bm{t}_{BC}]_{\times}\) is the skew-symmetric matrix of \(\bm{t}_{BC}\).

\subsection{Unified Modeling}
\label{unified model}

Based on the quadrotor model in Section~\ref{Qdynamics} and the image kinematics in Section~\ref{Ikinematics}, a unified predictive model is constructed for control design. By collecting the image features, translational velocity, and attitude states, the overall system is written as
\begin{equation}
\label{system model}
\dot{\bm{x}}=
\left[
\begin{array}{c}
\dot{\bm{\chi}} \\
\dot{\bm{v}}^{B} \\
\dot{\bm{\psi}}^{B}
\end{array}
\right]
=
\underbrace{
\left[
\begin{array}{c}
\vspace{0.2em}
\bm{J}\left[\bm{v}^{B};{}^{B}\bm{\omega}\right] \\
\vspace{0.4em}
\bm{R}_{WB}\dfrac{{}^{B}\bm{c}}{m}+\bm{g} \\
\bm{M}\,{}^{B}\bm{\omega}
\end{array}
\right]
}_{\bm{f}(\bm{x},\bm{u})}
+\bm{E}\bm{d},
\end{equation}
where $\bm{x}=\left[\bm{\chi};\bm{v}^{B};\bm{\psi}^{B}\right]\in\mathbb{R}^{8},~
\bm{u}=\left[{}^{B}\bm{\omega};c\right]\in\mathbb{R}^{4},$
\({}^{B}\bm{c}=(0,0,c)^\mathsf{T}\) is the collective thrust vector, \(\bm{d}\in\mathbb{R}^{5}\) is the lumped disturbance vector, and \(\bm{E}\in\mathbb{R}^{8\times5}\) is a constant disturbance distribution matrix.

The disturbance term \(\bm{d}\) is introduced to capture the lumped disturbances, including unmodeled aerodynamics, depth-estimation errors, camera-parameter mismatch, thrust-mapping uncertainty, and external disturbances such as wind. These uncertainties mainly affect the image-feature channels and the thrust-related translational dynamics. Accordingly, \(\bm{E}\) is chosen as $\bm{E}=
\begin{bmatrix}
\bm{I}_{5\times5}; 
\bm{0}_{3\times5}
\end{bmatrix},$
with \(\bm{I}_{5\times5}\) being the \(5\times5\) identity matrix. Since the lumped disturbance \(\bm{d}\) is unavailable for predictive control, its estimate \(\hat{\bm{d}}\), obtained from the prediction-enhanced ESKF introduced in Section~\ref{ESKF section}, is used in the control-oriented prediction model:
\begin{equation}
\label{final system model}
\dot{\bm{x}}=\bm{f}(\bm{x},\bm{u})+\bm{E}\hat{\bm{d}}.
\end{equation}


\section{ESKF with Image Feature Prediction}\label{ESKF section}

\begin{figure*}[htbp!]
    \centering
    \includegraphics[width=0.75\linewidth]{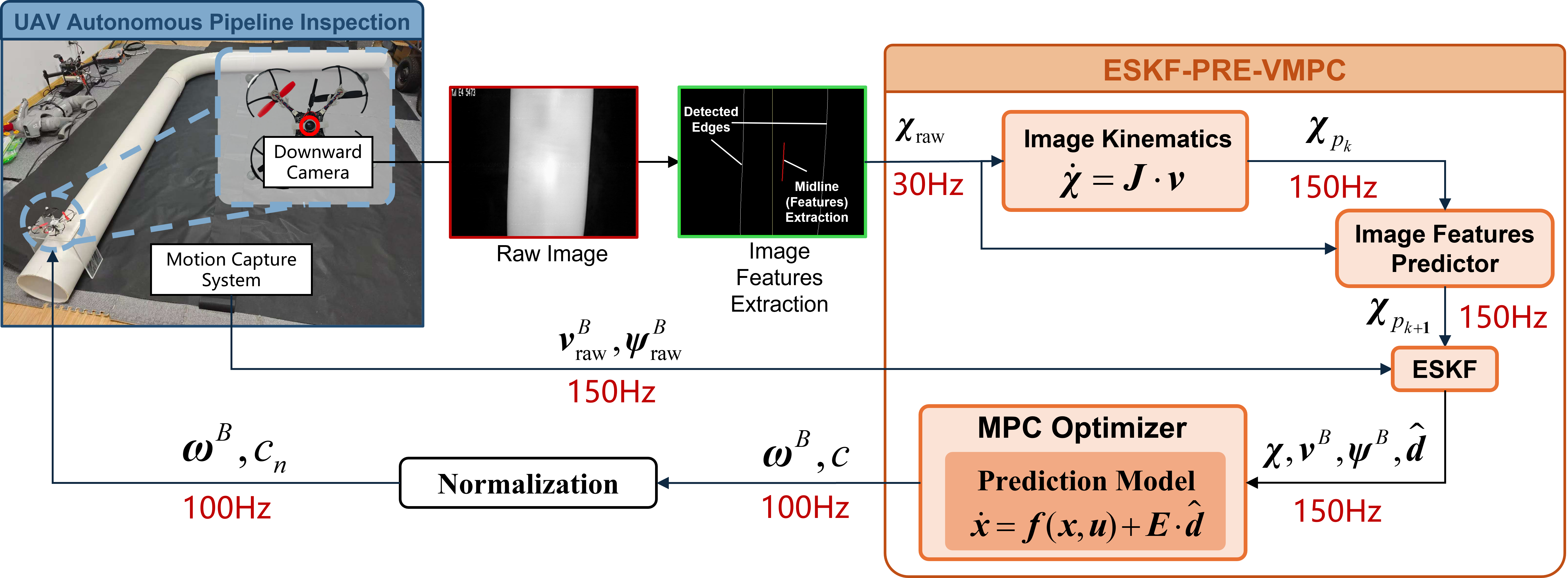}
    \caption{\textbf{Control framework of the proposed ESKF-PRE-VMPC for UAV autonomous pipeline inspection}. The framework integrates image feature extraction and prediction, state and lumped-disturbance estimation, and model predictive control. A downward-facing camera captures pipeline images, from which the pipeline edges and the midline are extracted. To address the low sampling rate of visual measurements, an image feature predictor based on image kinematics operates at the quadrotor state-measurement frequency. The predicted features are fused with available measurements in the ESKF to enable robust high-frequency estimation, providing timely state information, including image features, quadrotor states, and lumped disturbances, for MPC.}
    \label{fig:placeholder}
    \vspace{-0.4cm}
\end{figure*}

\subsection{ESKF Formulation}

Following Section~\ref{unified model}, disturbances are considered in the first five channels. Consequently, the subsystem associated with the first five states in \eqref{system model} can be rewritten as the following linear time-varying system:
\begin{equation}
\label{sub model}
\left\{
\begin{aligned}
\dot{\Tilde{\bm{x}}} &= \bm{A}_s \Tilde{\bm{x}} + \bm{B}_s \bm{u} + \bm{G} g + \bm{d},\\
\bm{y} &= \bm{C}\Tilde{\bm{x}},
\end{aligned}
\right.
\end{equation}
where \(\Tilde{\bm{x}}=\left[\bm{\chi}; \bm{v}^{B}\right] \in \mathbb{R}^{5}\), and \(\bm{y}\in\mathbb{R}^{5}\) denotes the measurement output. The system matrices are given by
\begin{equation}
\begin{aligned}
\bm{A}_s &=
\begin{bmatrix}
\bm{0}_{2\times2} & \bm{J}_{(:,1:3)} \\
\bm{0}_{3\times2} & \bm{0}_{3\times3}
\end{bmatrix}_{5\times5}, \qquad
\bm{G}=
\begin{bmatrix}
\bm{0}_{4\times1} \\
1
\end{bmatrix}_{5\times1},\\
\bm{B}_s &=
\begin{bmatrix}
\bm{J}_{(:,4:6)} & \bm{0}_{2\times1} \\
\bm{0}_{3\times3} & \dfrac{\bm{R}_{WB(:,3)}}{m}
\end{bmatrix}_{5\times4}, \qquad
\bm{C}=\bm{I}_{5\times5}.
\end{aligned}
\end{equation}

By incorporating the lumped disturbance into the extended-state representation and accounting for both process noise and measurement noise, \eqref{sub model} can be discretized as
\begin{equation}
\label{ESKF model}
\left\{
\begin{aligned}
\Bar{\bm{x}}_{k+1} &= \Bar{\bm{A}}_k \Bar{\bm{x}}_k + \Bar{\bm{B}}_k \bm{u}_k + \Bar{\bm{G}} g + \bm{L}\bm{D}_k + \bm{W}\bm{w}_k, \\
\bm{y}_k &= \Bar{\bm{C}} \Bar{\bm{x}}_k + \bm{n}_k,
\end{aligned}
\right.
\end{equation}
\begin{equation}\nonumber
\begin{aligned}
    \Bar{\bm{x}}_{k} &= \begin{bmatrix}
    \Tilde{\bm{x}}_k \\
    \bm{d}_k
    \end{bmatrix}\in\mathbb{R}^{10}, \
    \Bar{\bm{A}}_k =\begin{bmatrix}
    \bm{I}+{t}_s\bm{A_s} & t_s\bm{I} \\
    \bm{0} & \bm{I}
\end{bmatrix}_{10\times10}, \\[0.5em]
\Bar{\bm{B}}_k&=\begin{bmatrix}
    t_s\bm{B_s} \\
    \bm{0}
\end{bmatrix}_{10\times4},\
\Bar{\bm{G}}=\begin{bmatrix}
    t_s\bm{G} \\
    \bm{0}
\end{bmatrix}\in\mathbb{R}^{10},\
\bm{L}=\begin{bmatrix}
    \bm{0} \\
    \bm{I}
\end{bmatrix}_{10\times5},\\[0.5em]
\bm{W}&=\begin{bmatrix}
    \bm{I} \\
    \bm{0}
\end{bmatrix}_{10\times5},
\bm{D}_k=\bm{d}_{k+1}-\bm{d}_k,\,
\Bar{\bm{C}}=\begin{bmatrix}
    \bm{C} & \bm{0}
\end{bmatrix}_{5\times10}.
\end{aligned}
\end{equation}
where \(\bm{w}_k\in\mathbb{R}^{5}\) and \(\bm{n}_k\in\mathbb{R}^{5}\) denote the process noise and measurement noise, respectively. The measurement vector \(\bm{y}_k\) is formed from the available image features and the measured linear-velocity states at time \(k\).

To construct the ESKF as described in \cite{bai2018extended}, the following assumptions are made.

\begin{assumption}
\(\bm{w}_k\) and \(\bm{n}_k\) are uncorrelated zero-mean white noise processes, with \(E(\bm{n}_k\bm{n}_k^\mathsf{T})\leq\bm{R}_k\) and \(E(\bm{w}_k\bm{w}_k^\mathsf{T})\leq\bm{S}_k\), where \(\bm{R}_k\in\mathbb{R}^{5\times5}\) is the measurement-noise covariance and \(\bm{S}_k\in\mathbb{R}^{5\times5}\) is the process-noise covariance.
\end{assumption}

\begin{assumption}
\label{observability}
\((\Bar{\bm{A}}_k,\Bar{\bm{C}})\) is uniformly observable.
\end{assumption}

\begin{assumption}
\(\bm{P}_0\) is the initial state-estimation error covariance matrix, satisfying
\(E\!\left([\Bar{\bm{x}}_0-\hat{\Bar{\bm{x}}}_0][\Bar{\bm{x}}_0-\hat{\Bar{\bm{x}}}_0]^\mathsf{T}\right)\leq\bm{P}_0\),
where \(\hat{\Bar{\bm{x}}}_k=[\hat{\Tilde{\bm{x}}}_k;\hat{\bm{d}}_k]\) is the estimate of \(\Bar{\bm{x}}_k\).
\end{assumption}

\begin{assumption}
\(E(\bm{D}_{k,i}^{2})\leq q_{k,i}\), \(i=1,2,\ldots,5\), \(\forall k\geq0\), where \(q_{k,i}\) is uniformly bounded.
\end{assumption}

\begin{remark}
Assumption~\ref{observability} holds under normal operation, where the feature Jacobian \(\bm{J}\) remains full-rank, the image features \(\bm{\chi}\) provide continuous feedback, and the UAV motion provides persistent excitation.
\end{remark}

Following the extended-state representation, the ESKF algorithm is constructed based on \cite{bai2018extended} as
\begin{equation}
\label{ESKF}
\begin{aligned}
\hat{\Bar{\bm{x}}}_{k+1} &= \Bar{\bm{A}}_k \hat{\Bar{\bm{x}}}_k + \Bar{\bm{B}}_k \bm{u}_k + \Bar{\bm{G}} g + \bm{L}\hat{\bm{D}}_k - \bm{K}_k \big(\bm{y}_k - \Bar{\bm{C}} \hat{\Bar{\bm{x}}}_k\big), \\
\bm{K}_k &= -\Bar{\bm{A}}_k \bm{P}_k \Bar{\bm{C}}^\mathsf{T}
\left( \Bar{\bm{C}} \bm{P}_k \Bar{\bm{C}}^\mathsf{T} + (1+\alpha)^{-1} \bm{R}_k \right)^{-1}, \\
\bm{P}_{k+1} &= (1+\alpha)(\Bar{\bm{A}}_k + \bm{K}_k \Bar{\bm{C}})\bm{P}_k(\Bar{\bm{A}}_k + \bm{K}_k \Bar{\bm{C}})^\mathsf{T} \\
&\quad + \bm{K}_k \bm{R}_k \bm{K}_k^\mathsf{T} + \bm{Q}_{1,k} + \bm{Q}_{2,k}, \\
\bm{Q}_{1,k} &=
\begin{bmatrix}
\bm{0}_{5\times5} & \bm{0}_{5\times5} \\
\bm{0}_{5\times5} & 4\bm{Q}_k
\end{bmatrix}, \qquad
\bm{Q}_{2,k} =
\begin{bmatrix}
\bm{S}_k & \bm{0}_{5\times5} \\
\bm{0}_{5\times5} & \bm{0}_{5\times5}
\end{bmatrix},\\
\bm{Q}_k &= 5\cdot\mathrm{diag}([q_{k,1},\ldots,q_{k,5}]),
\end{aligned}
\end{equation}
\begin{equation}\nonumber
\alpha=\sqrt{\frac{\operatorname{tr}(\bm{Q}_{1,0})}{\operatorname{tr}(\bm{P}_0)}},
\qquad
\hat{\bm{D}}_{k}\triangleq\bm{0}.
\end{equation}
\(\hat{\bm{D}}_{k}\) denotes the estimate of \(\bm{D}_k\), and \(\bm{K}_k\) is the Kalman gain.

\begin{remark}
No prior knowledge of the lumped disturbance \(\bm{d}_k\) is assumed, leading to the nominal model \(\bm{d}_{k+1}-\bm{d}_{k}=\bm{0}\). If additional information about \(\bm{d}_k\) is available, a more refined nominal model can be adopted to improve the accuracy of \(\hat{\bm{D}}_{k}\). More detailed derivations can be found in \cite{bai2018extended}.
\end{remark}

\begin{lemma}
\label{asymptotically unbiased theorem}
If the lumped disturbance \(\bm{d}_k\) is constant, i.e., \(\bm{d}_{k+1}-\bm{d}_{k}=\bm{0}\), then the ESKF in \eqref{ESKF} provides an asymptotically unbiased minimum-variance estimate of \(\Bar{\bm{x}}_k\), 
\[
\lim_{k\to\infty}E(\Bar{\bm{e}}_{k}\Bar{\bm{e}}_{k}^{\mathsf{T}})=\bm{P}_{\infty},\qquad
\lim_{k\to\infty}E(\Bar{\bm{e}}_{k})=\bm{0},
\]
where \(\Bar{\bm{e}}_k=\Bar{\bm{x}}_k-\hat{\Bar{\bm{x}}}_k\) is the estimation error, and \(\bm{P}_{\infty}=\lim_{k\to\infty}\bm{P}_k\) is a constant positive definite matrix.
\end{lemma}

\begin{proof}
The proof of Lemma ~\ref{asymptotically unbiased theorem} is detailed in \cite{bai2018extended}.
\end{proof}

\begin{remark}
Although Lemma ~\ref{asymptotically unbiased theorem} is established under the assumption of constant lumped disturbance, the ESKF remains effective in practice under slowly time-varying uncertainties, which makes it suitable for most real-world scenarios.
\end{remark}

\vspace{-0.1cm}
\subsection{Image Feature Prediction}

Due to the low camera update rate, the extracted image features cannot provide high-frequency measurements to the ESKF, resulting in low-frequency state estimation. To address this issue, open-loop prediction of the image features is first performed using the computed feature Jacobian and high-frequency state information. Specifically, the high-rate state-measurement timestamps are used as the reference. When no new image is received or the visual detection is unreliable, the image features are predicted; when a new reliable image becomes available, it is used to correct the predicted features:
\begin{equation}
\label{image prediction}
\bm{\chi}_{p,k+1}=
\begin{cases}
\bm{\chi}_{k+1}, & k\in\{k_j\}, \\
\bm{\chi}_{p,k} + T_s \bm{J}_k \bm{\nu}_{k}, &\text{otherwise},\\
\end{cases}
\end{equation}
where \(\bm{\chi}_{p,k}\) is the prediction of \(\bm{\chi}\) at time step \(k\), and \(\bm{\nu}_k\) is obtained from the available high-rate state measurements and control inputs. The Jacobian matrix \(\bm{J}_k\) is updated only when new image data are received and is assumed to remain constant between two consecutive image frames. This prediction and estimation process is illustrated in Fig.~\ref{prediction}.
\begin{figure}[htbp!]
    \centering
    \includegraphics[width=0.9\linewidth]{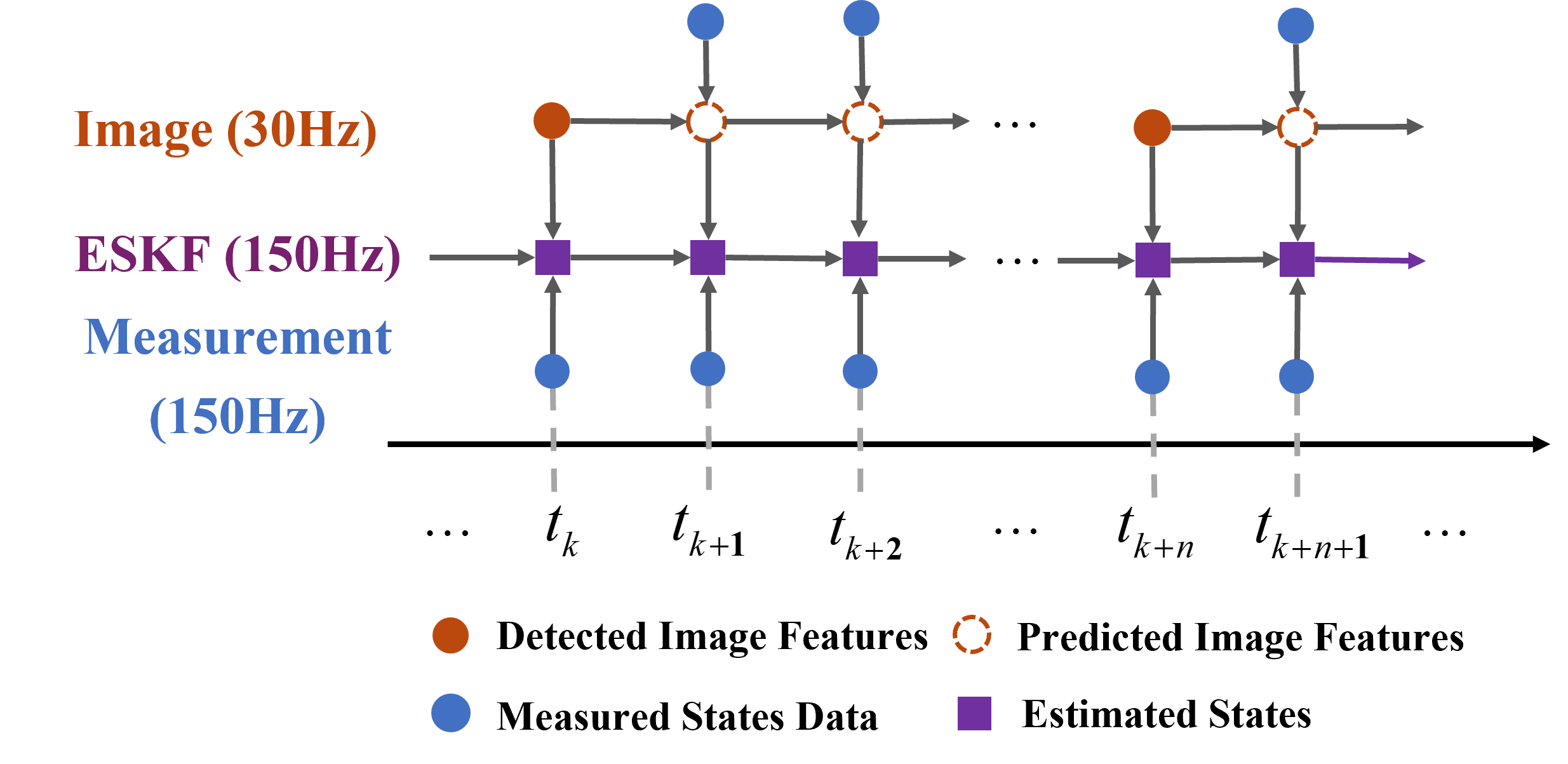}
    \caption{\textbf{ESKF estimator} with image feature prediction.}
    \label{prediction}
    \vspace{-0.1cm}
\end{figure}

\begin{remark}
Although the image features are first predicted in an open-loop manner, these predictions are treated as pseudo-measurements in the ESKF. The final image features used by the MPC are obtained from the ESKF estimates, thereby forming a closed-loop image-feature estimation scheme.
\end{remark}

Specifically, the measurement-noise covariance \(\bm{R}_k\) in the ESKF is designed to be adaptive with respect to the source of visual features. For features directly detected from the current image, a lower noise level is assigned, whereas for predicted features, a higher noise level is used:
\begin{equation}
\bm{R}_k=
\begin{cases}
\bm{R}_0, & \bm{\chi}_{p,k+1}=\bm{\chi}_{k+1},\\
\bm{R}_1, & \bm{\chi}_{p,k+1}=\bm{\chi}_{p,k}+T_s\bm{J}_k\bm{\nu}_{k},
\end{cases}
\quad \bm{R}_1 \succeq \bm{R}_0.
\end{equation}

\section{Visual Servoing MPC}

\subsection{VMPC Design}

The VMPC controller computes the control commands by solving a finite-horizon optimal control problem. The primary objective is to perform visual servoing while maintaining the desired inspection velocity. To this end, an intermediate state is defined as
\begin{equation}
    v_h = \cos(\alpha)\, v_x^B + \sin(\alpha)\, v_y^B,
\end{equation}
which represents the horizontal velocity component along the UAV heading direction. In addition, the pipeline should remain at an appropriate size in the image to support stable and high-quality inspection. The vertical velocity \(v_z^B\) is used to regulate the apparent pipeline size in the image. Accordingly, the VMPC optimization state is chosen as
\begin{equation}
    \bm{x}_o =
    \begin{bmatrix}
        \bm{\chi};
        v_h;
        v_z^B
    \end{bmatrix}.
\end{equation}

To reduce the risk of IBVS divergence caused by positive feedback~\cite{ozawa2011dynamic}, it is desirable to keep the commanded body angular velocities, especially \({}^{B}\omega_x\) and \({}^{B}\omega_y\), as small as possible. Since the collective thrust is not directly optimized in the VMPC, the optimization input is defined as
\begin{equation}
    \bm{u}_o =
    \begin{bmatrix}
        {}^{B}\omega_x;
        {}^{B}\omega_y;
        {}^{B}\omega_z
    \end{bmatrix}.
\end{equation}

By discretizing the control-oriented prediction model in \eqref{final system model} using a fourth-order Runge--Kutta method with sampling time \(t_p\), and incorporating the disturbance estimate \(\hat{\bm{d}}_k\) provided by the ESKF, a discrete-time prediction model for VMPC is obtained as
\begin{equation}
    \bm{x}_{k+1}
    =
    \bm{f}_o\!\left(
    \bm{x}_{k}, \bm{u}_{k}; \hat{\bm{d}}_k
    \right).
\end{equation}

The nonlinear VMPC problem over the prediction horizon \(N\) is formulated as
\begin{equation}
\label{optimization}
\begin{aligned}
    \min_{\bm{u}_{o,0},\ldots,\bm{u}_{o,N-1}} \;\;
    & \sum_{i=0}^{N-1}
    \Big(
    \|\Delta\bm{x}_o(k+i|k)\|_{\bm{\Phi}}^2+
    \|\Delta\bm{u}_o(k+i|k)_{\bm{\Psi}}^2
    \Big)   \\
    & \qquad+\|\Delta\bm{x}_o(k+N|k)\|_{\bm{\Phi}_N}^2 \\
    \mathrm{s.t.}\;\;
     \bm{x}(k+i&+1|k)
    =
    \bm{f}_o\!\left(
    \bm{x}(k+i|k), \bm{u}(k+i|k); \hat{\bm{d}}_k
    \right),
\end{aligned}
\end{equation}
where $ \Delta\bm{x}_{o,k} = \bm{x}_{o,k} - \bm{x}_{o_d,k}, \Delta\bm{u}_{o,k} = \bm{u}_{o,k} - \bm{u}_{o_d,k}$, 
with \(\bm{x}_{o_d,k}\) and \(\bm{u}_{o_d,k}\) the desired state and input vectors, respectively. \(\bm{\Phi}\) and \(\bm{\Psi}\) are positive semi-definite weight matrices that specify the relative priorities of the optimization variables.


In this work, the desired input vector is set to \(\bm{u}_{o_d,k}=\bm{0}\), while the desired state vector is defined as
\begin{equation}
\label{xod}
    \bm{x}_{o_d,k}
    =
    \begin{bmatrix}
        \bm{\chi}_d;
        v_{h_d};
        v_{z_d,k}^B
    \end{bmatrix},
\end{equation}
where \(\bm{\chi}_d=(\theta_d,r_d)^{\mathsf T}=(\pi/2,0)^{\mathsf T}\) ensures that the pipeline remains vertically aligned and centered in the image, \(v_{h_d}\) is the desired inspection velocity, and \(v_{z_d,k}^B\) is designed as
\begin{equation}
\label{vzd}
    v_{z_d,k}^B
    =
    k_1\!\left(
    \frac{1}{\rho_d} - \frac{1}{\rho_k}
    \right)
    +
    k_2
    \frac{\rho_{f,k}-\rho_k}{c_v}
    v_{h_d},
\end{equation}
where \(\rho_d\) denotes the desired pipeline width at the image center, \(\rho_k\) and \(\rho_{f,k}\) denote the detected pipeline widths at the image center and image front at time step \(k\), respectively, and \(k_1\) and \(k_2\) are tunable parameters. The first term in \eqref{vzd} regulates the apparent pipeline size based on the inverse proportionality between flight altitude and pipeline width in the image. The second term accounts for sloped terrain by exploiting the width difference between the front and center of the image, thereby generating a velocity adjustment for both upward and downward slopes.

To analyze recursive feasibility and closed-loop stability, define
$\bm{e}_k = \Delta\bm{x}_{o,k}, \bm{\delta}_k = \Delta\bm{u}_{o,k}.$
Then the cost function in \eqref{optimization} can be rewritten as
\begin{equation}
\label{cost function}
\begin{aligned}
    J(k)
    =\;
    \sum_{i=0}^{N-1}
    &\Big(
    \|\bm{e}(k+i|k)\|_{\bm{\Phi}}^2
    +
    \|\bm{\delta}(k+i|k)\|_{\bm{\Psi}}^2
    \Big) \\
    &+
    \|\bm{e}(k+N|k)\|_{\bm{\Phi}_N}^2,
\end{aligned}
\end{equation}
subject to
\begin{equation}
    \bm{e}(k+i+1|k)
    =
    h\!\left(
    \bm{e}(k+i|k), \bm{\delta}(k+i|k)
    \right),
\end{equation}
where \(h\) denotes the virtual mapping that describes the error evolution under the corresponding optimization input.

\begin{assumption}
\label{terminal cost}
For any \(\bm{e}\), there exists a control input \(\bm{\delta}(\bm{e})\) such that
\begin{equation}
\label{terminal_ineq}
\begin{aligned}
    \|h(\bm{e}, \bm{\delta}(\bm{e}))\|_{\bm{\Phi}_N}^2
    \leq\;&
    \|\bm{e}\|_{\bm{\Phi}_N}^2
    - \|\bm{e}\|_{\bm{\Phi}}^2- \|\bm{\delta}(\bm{e})\|_{\bm{\Psi}}^2.
\end{aligned}
\end{equation}
\end{assumption}

\begin{theorem}
\label{stability}
If the optimization problem \eqref{optimization} is feasible at \(k=0\), then the VMPC controller is recursively feasible, and the closed-loop system is asymptotically stable.
\end{theorem}

\begin{proof}
The optimization problem \eqref{optimization} is equivalent to minimizing the cost function \eqref{cost function}:
\begin{equation}
\label{optimize J}
    \min_{\bm{\delta}(k|k),\ldots,\bm{\delta}(k+N-1|k)} J(k).
\end{equation}

Suppose that at time step \(k\), the optimal control sequence obtained by solving \eqref{optimize J} is
\(\bm{\delta}^*(k|k), \ldots, \bm{\delta}^*(k+N-1|k)\), with the corresponding state sequence
\(\bm{e}^*(k+1|k), \ldots, \bm{e}^*(k+N|k)\). The optimal cost \(J^*(k)\) is then
\begin{equation}
\begin{aligned}
    J^*(k)
    =\;&
    \sum_{i=0}^{N-1}
    \Big(
    \|\bm{e}^*(k+i|k)\|_{\bm{\Phi}}^2
    +
    \|\bm{\delta}^*(k+i|k)\|_{\bm{\Psi}}^2
    \Big) \\
    &+
    \|\bm{e}^*(k+N|k)\|_{\bm{\Phi}_N}^2.
\end{aligned}
\end{equation}

At time step \(k+1\), consider the shifted control sequence
\[
\bm{\delta}^*(k+1|k),\ldots,\bm{\delta}^*(k+N-1|k),
\bm{\delta}(\bm{e}^*(k+N|k)),
\]
with the corresponding state sequence
\[
\bm{e}^*(k+2|k),\ldots,\bm{e}^*(k+N|k),
h\!\left(
\bm{e}^*(k+N|k), \bm{\delta}(\bm{e}^*(k+N|k))
\right).
\]
Since \(\bm{e}(k+1|k+1)=\bm{e}^*(k+1|k)\), this provides a feasible candidate solution at time \(k+1\). Therefore, if the optimization problem is feasible at \(k=0\), it remains feasible for all \(k>0\), which proves recursive feasibility.

Meanwhile, the cost corresponding to this feasible solution at \(k+1\) satisfies
\begin{equation}
\begin{aligned}
    J(k+1) &- J^*(k) 
    = \|h\big(\bm{e}^*(k+N|k), \bm{\delta}(\bm{e}^*(k+N|k))\big)\|_{\bm{\Phi}_N}^2 \\ 
    &+ \|\bm{\delta}(\bm{e}^*(k+N|k))\|_{\bm{\Psi}}^2 
    + \|\bm{e}^*(k+N|k)\|_{\bm{\Phi}}^2 \\
     &- \|\bm{e}^*(k+N|k)\|_{\bm{\Phi}_N}^2 
    - \|\bm{e}^*(k|k)\|_{\bm{\Phi}}^2 
    - \|\bm{\delta}^*(k|k)\|_{\bm{\Psi}}^2
\end{aligned}\label{costdiff}
\end{equation}
By Assumption~\ref{terminal cost}, \eqref{costdiff} yields
\begin{equation}
\label{decrease}
\begin{aligned}
    J^*(k+1)-J^*(k)
    &\leq J(k+1)-J^*(k) \\
    &\leq
    -\|\bm{e}^*(k|k)\|_{\bm{\Phi}}^2
    -\|\bm{\delta}^*(k|k)\|_{\bm{\Psi}}^2.
\end{aligned}
\end{equation}
By selecting \(J^*(k)\) as a Lyapunov function candidate, 
\eqref{decrease} implies that \(J^*(k)\) is monotonically decreasing, which ensures asymptotic stability of the closed-loop system.
\end{proof}

\section{Experimental Validation}
\label{exp}

The effectiveness of the proposed framework is validated via both high-fidelity Gazebo simulations and real-world experiments.
In all experiments, the MPC optimization problem is solved online using ACADO~\cite{houska2011acado} with the qpOASES solver~\cite{ferreau2014qpoases}. For a detailed visualization of the results, see the project webpage.
To isolate the impact of disturbance estimation and image feature prediction, three representative baselines are selected for comparison.

\subsection{Comparative Algorithms}

The proposed ESKF-PRE-VMPC framework for quadrotor visual servoing is compared against three baseline algorithms, which are summarized below for  clarity: 

\begin{enumerate}

    \item [(1)] \textbf{IBVS}: The classical IBVS is first considered, 
    without accounting for quadrotor dynamics. Due to the rank deficiency of the visual feature Jacobian matrix, a null-space projection operator is employed to facilitate tracking of the desired cruise velocity, as detailed in \cite{dietrich2015overview, velasco2024visual}. Control commands sent to the PX4 consist of 6-DOF velocities, including linear velocity \(\bm{v}^B\) and angular velocity \(^B\bm{\omega}\);
    
    \item [(2)] \textbf{IBVS-MPC}: In this approach, the MPC model only accounts for the image feature kinematics as described in Section~ \ref{Ikinematics} instead of the quadrotor dynamics, and similar to IBVS, the generated control commands directly provide 6-DOF velocity inputs to the autopilot, corresponding to the integrated VMPC-K method in Table~\ref{comparison}. 
    
    \item [(3)] \textbf{ESKF-VMPC}: This algorithm uses the same system model, state estimation and disturbance compensation strategy, and MPC problem formulation as the proposed ESKF-PRE-VMPC. However, it does not perform prediction or estimation of the image features, and thus updates only when new reliable image information is received.

\end{enumerate}

In the comparative experiments, the necessary parameters for control algorithm design and implementation are summarized in Table \ref{tab_exp_params}. 
\begin{table}[htbp!]
\vspace{-0.1cm}
\centering
\caption{Control parameters in experiments. (a) Parameters with the same values. (b) Parameters with different values, where S denotes simulations and R denotes real-world experiments.}
\label{tab_exp_params}
\renewcommand{\arraystretch}{1.1}
\resizebox{\columnwidth}{!}{%
\begin{tabular}{cc|cc}
\multicolumn{4}{c}{\textbf{(a) Shared parameters}}\\
\toprule
Parameter & Value & Parameter & Value \\
\midrule
\(N\) & 20 & \(t_p\) & 0.1 s \\
\(\rho_d\) & 300 pixels & \(\bm{R}_0\) & \(\mathrm{diag}(0.01,\,0.01,\,0.01,\,0.01,\,0.01)\) \\
\(\bm{R}_1\) & \(\mathrm{diag}(0.1,\,0.1,\,0.01,\,0.01,\,0.01)\) & \(\bm{S}_k\) & \(0.001\bm{I}_{5\times5}\) \\
\(\bm{P}_0\) & \(\mathrm{diag}(0.5\bm{I}_{5\times5},\,\bm{I}_{5\times5})\) & \(\bm{\Psi}\) & \(\mathrm{diag}(0.05,\,0.05,\,0.01)\) \\
\bottomrule
\end{tabular}}
\vspace{0.5em}
\resizebox{\columnwidth}{!}{%
\begin{tabular}{c|cc}
\multicolumn{3}{c}{\textbf{(b) Scenario-dependent parameters}}\\
\toprule
Parameter & Value (S) & Value (R) \\
\midrule
\(m\) & 1.5 kg & 0.044 kg \\
\(v_{h_d}\) & 0.5 m/s & 0.3 m/s \\
\(D\) & 1.5 m & 0.2 m \\
\(\bm{\Phi}\) & \(\mathrm{diag}(\phi_\theta,\,\phi_r,\,0.01,\,1)\) & \(\mathrm{diag}(0.3,\,0.1,\,0.01,\,0.1)\) \\
\(\bm{\Phi}_N\) & \(\bm{\Phi}\) & \(\bm{\Phi}\) \\
\bottomrule
\end{tabular}}
\vspace{-0.2cm}
\end{table}

\subsection{SITL simulation}

\subsubsection{Simulation setup}
We conducted Software-In-The-Loop (SITL) simulations in a high-fidelity \href{https://classic.gazebosim.org/}{Gazebo} environment integrated with PX4 and ROS. A detailed pipeline model, as shown in the bottom plot of Fig.~\ref{3D model}, was carefully designed to include various practical scenarios, such as straight and curved segments, uphill and downhill segments, representing the real-world exemplary pipelines (top plot of Fig.~\ref{3D model}). 
\begin{figure}[htbp!]
    \centering
    \includegraphics[width=0.8\linewidth, height=0.3\linewidth]{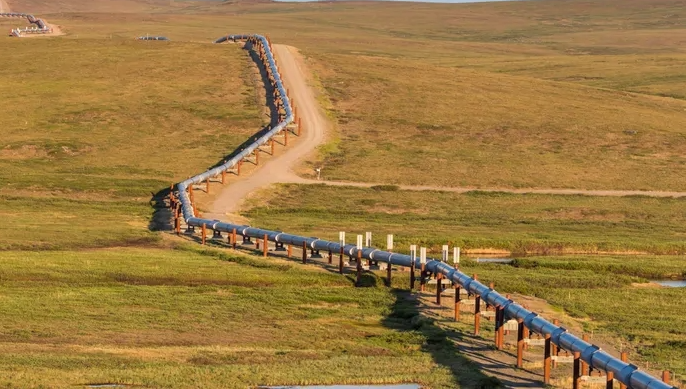}
    \includegraphics[width=0.8\linewidth, height=0.38\linewidth]{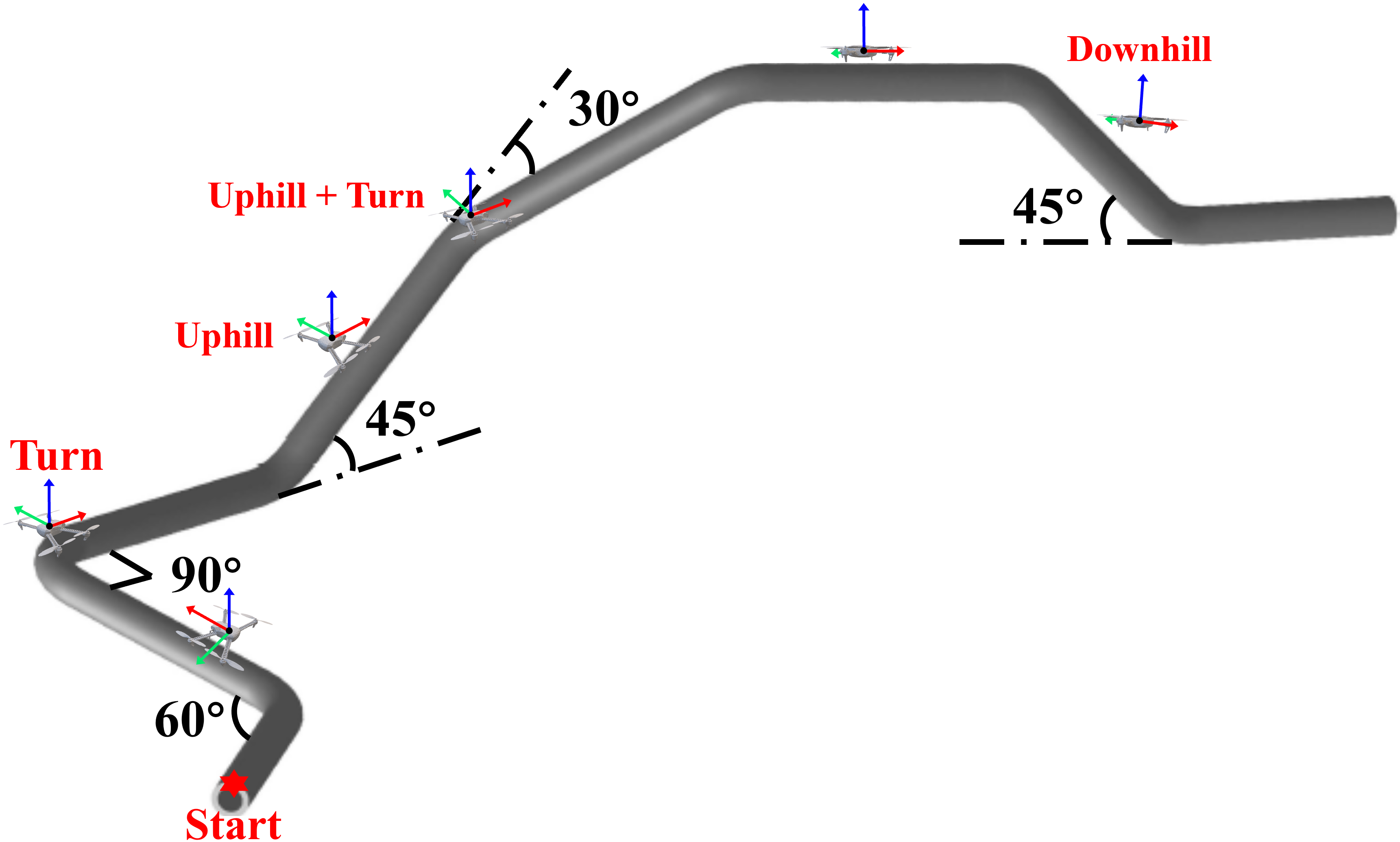}
    \caption{\textbf{Pipeline model} is composed of eight segments with each 8 meters long. Top: an exemplary real-world \href{https://www.spiegel.de/wissenschaft/natur/alaskas-umstrittene-bohrprojekte-biden-versucht-s-noch-mal-a-a1785085-0cfa-4805-8406-4fbdb3f32c1f}{pipeline}; Bottom: schematic of the 3D pipeline in this work, showing the curvature and slope of each segment.}
    \label{3D model}
    \vspace{-0.1cm}
\end{figure}

The quadrotor physics model is based on the \href{https://www.3dr.com/}{3DR Iris Quadrotor}. The quadrotor is equipped with a downward-facing camera for visual feature acquisition, featuring a resolution of \(848 \times 480\)  with known intrinsic parameters. All simulations were conducted on an Intel i5-13600KF CPU @ 5.1 GHz. Two sets of comparative simulations are conducted:
\begin{itemize}  
    \item \textbf{Wind-free case}: No wind is introduced in Gazebo. The uncertainties arise from modeling errors and unknown thrust mapping relationship; 

    \item \textbf{Wind case}: A continuous wind with a mean speed of \(3~\text{m/s}\) and a variance of \(0.04~\text{(m/s)}^2\) is applied along the \(\vec{y}_W\) direction.
\end{itemize}  

In simulation, the measurements are obtained from the onboard camera and the Gazebo ground truth. To assess the performance of the estimator, zero-mean Gaussian white noise $\mathcal{N}(\bm{0},\sigma^2\bm{I})$ is added to the ground truth measurements from Gazebo when running the ESKF-VMPC and ESKF-PRE-VMPC algorithms.

\subsubsection{Simulation Results}
As shown in Figs.~\ref{fig:rmse} and ~\ref{fig:evaluation}, the proposed framework consistently achieves the lowest overall image features regulation errors across both scenarios. Additionally, in the wind case, where both IBVS and IBVS-MPC fail, ESKF-based methods successfully complete the inspection task. Besides the RMSE index, it can also be observed from Fig.~\ref{fig:evaluation} that the proposed method has the smoothest curves, which means more stable and higher-quality images, which is precisely the purpose of the inspection mission.


Notably, when the variance of the measurement noise increases from \(\sigma^2=0.001\) to \(0.01\), the RMSE of ESKF-VMPC increases by \(87.57\%\) in \(r\) and \(26.06\%\) in \(\theta\) in the wind-free case, and by \(10.37\%\) in \(r\) and \(24.02\%\) in \(\theta\) in the wind case. In contrast, ESKF-PRE-VMPC changes by only \(1.43\%\) in \(r\) and \(8.33\%\) in \(\theta\) in the wind-free case, and by \(2.30\%\) in \(r\) and \(-4.13\%\) in \(\theta\) in the wind case. This indicates that the proposed method is significantly less sensitive to measurement-noise degradation, which is also consistent with the comparably smooth curves in Fig.~\ref{fig:evaluation}.
\begin{figure}[htbp!]
    \centering
    \includegraphics[width=0.95\linewidth]{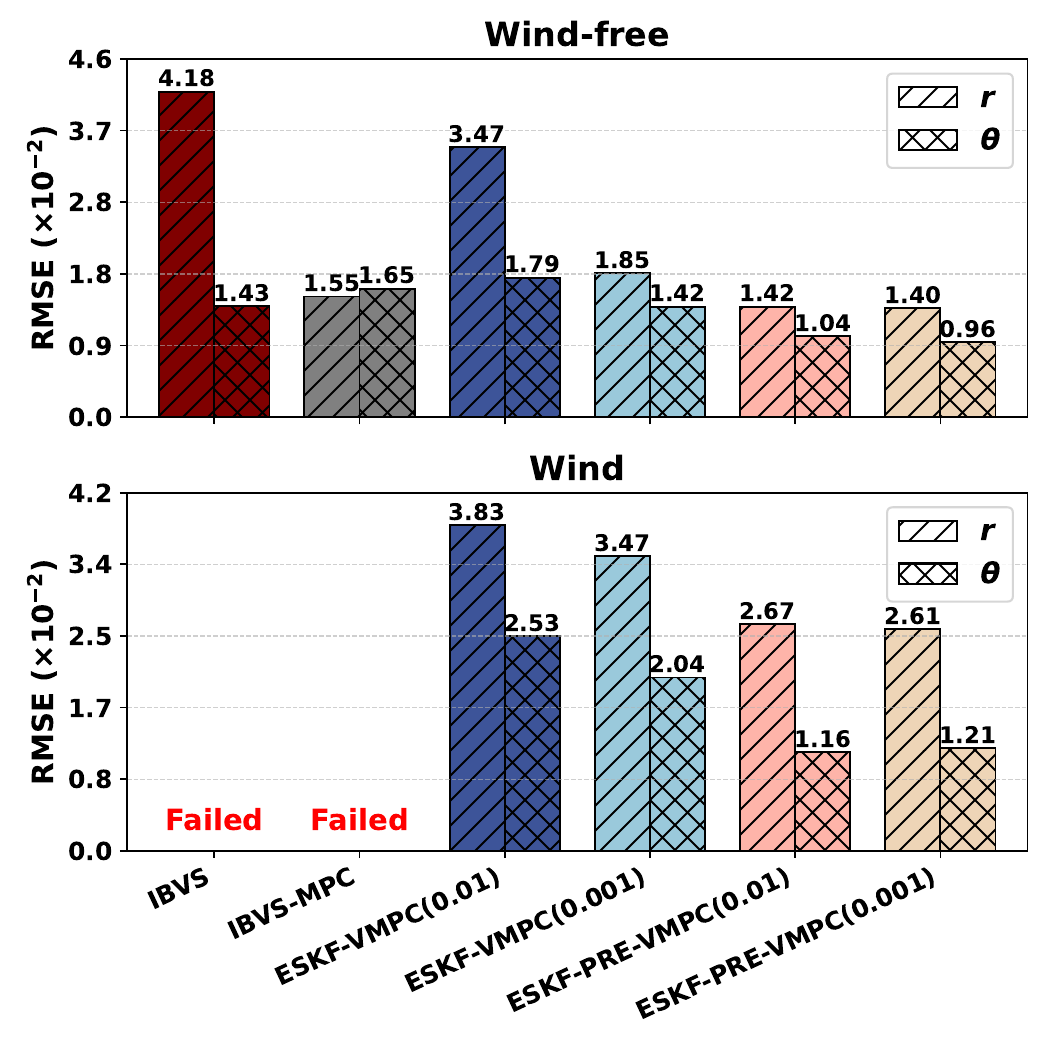}
    \caption{\textbf{Image features RMSE} (lower is better): the values in parentheses after methods indicate the variance $\sigma^2$ of the added zero-mean Gaussian noise.}
    \label{fig:rmse}
    \vspace{-0.1cm}
\end{figure}

\begin{figure}[htbp!]
\centering
\subfloat[\textbf{Wind-free case}]{%
    \centering
    \includegraphics[width=\linewidth]{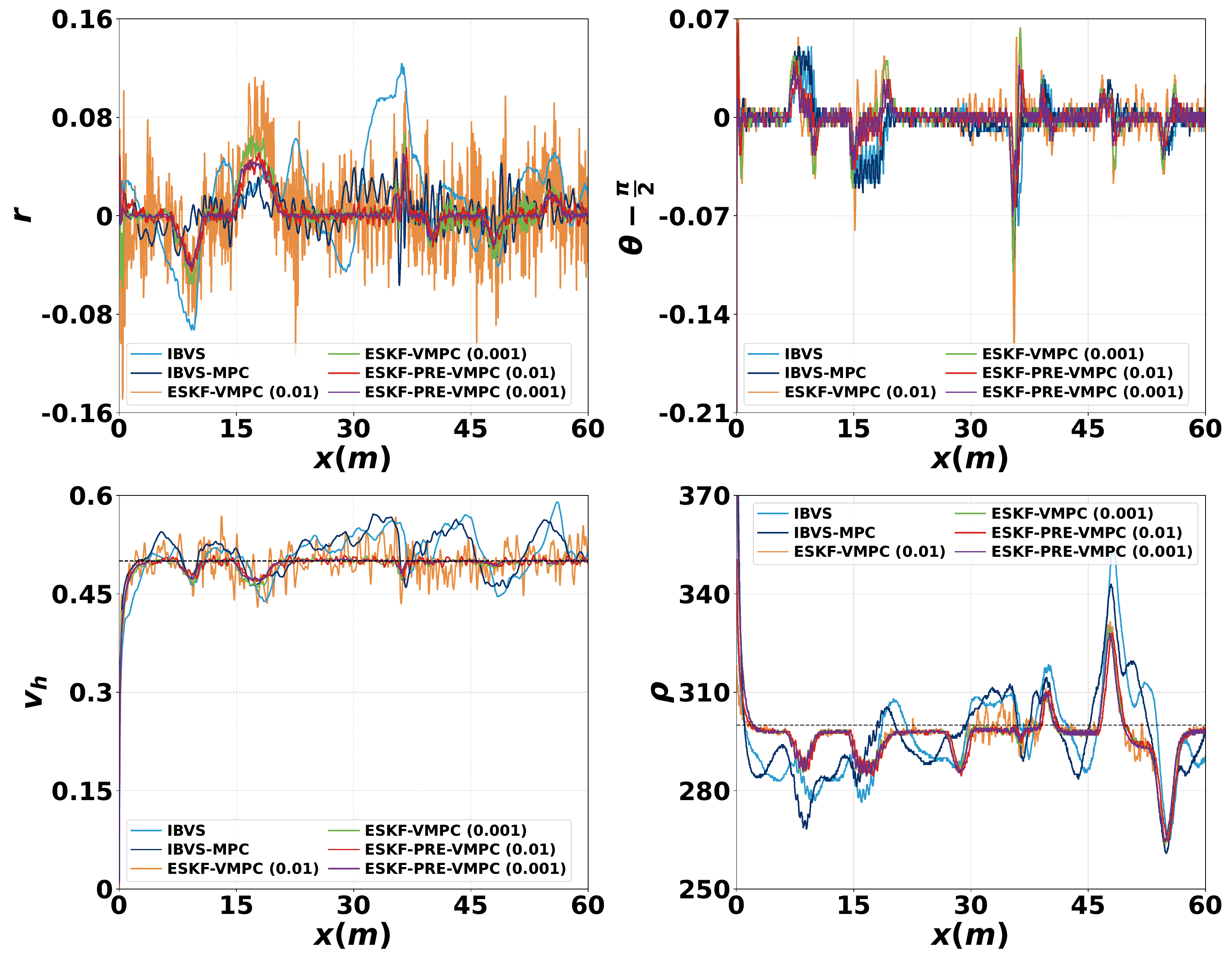}%
    \label{fig:eval_a}%
}
\vspace{0.1em}
\subfloat[\textbf{Wind case}]{%
    \centering
    \includegraphics[width=\linewidth]{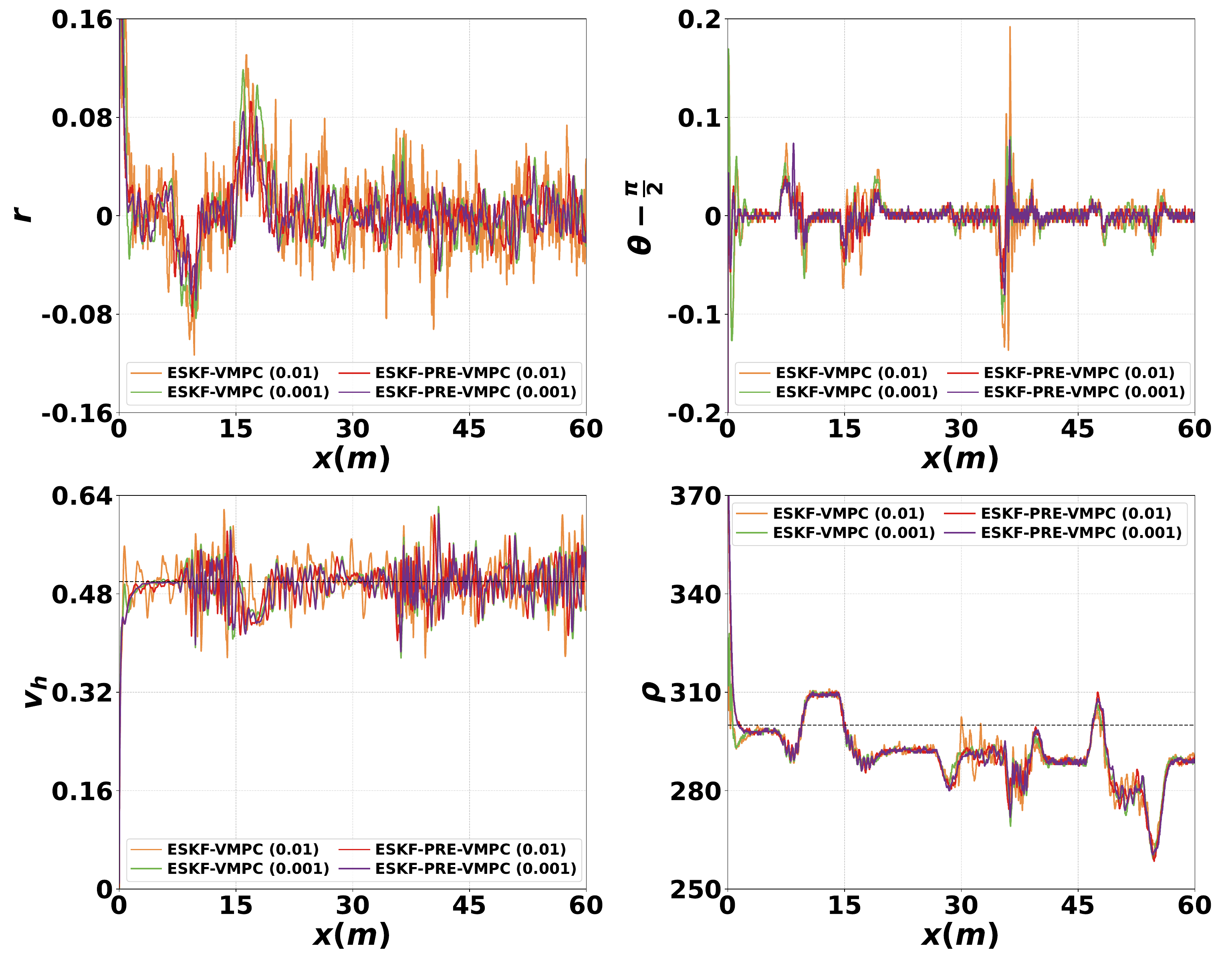}%
    \label{fig:eval_b}%
}
\caption{\textbf{Comparative results} under wind-free and wind cases.}
\label{fig:evaluation}
\vspace{-0.1cm}
\end{figure}

Moreover, a key point that is not directly reflected in the plots is that, during the comparison, the MPC-based methods do \emph{not} use the same weighting matrix $\bm{\Phi}$. Instead, we consider two representative parameter sets as summarized in Table~\ref{tab:q_settings}, and the rationale is explained as follows. In pipeline inspection, due to the underactuated nature of quadrotors, regulating $r$ is significantly more challenging than regulating $\theta$. Specifically, the control of $r$ relies mainly on the translational motion along the ${\vec{y}}_B$ direction, which is strongly coupled with the roll dynamics and introduces a positive feedback effect~\cite{ozawa2011dynamic}. In contrast, the regulation of $\theta$ primarily depends on the yaw angular velocity command $w_z$, which is largely decoupled from the other degrees of freedom, and consequently results in better performance as shown in the plots. Therefore, it is intuitive to assign a larger weight to $r$ so that the optimizer prioritizes $r$ regulation. However, excessively large weights on $r$ tend to generate aggressive control commands, which can easily push the closed-loop system of the underactuated platform toward instability. This motivates a trade-off between control accuracy and stability.

Accordingly, we categorize the MPC weights into two groups according to the magnitude of the penalty on $r$: an \emph{aggressive} setting (larger $r$ weight) and a \emph{conservative} setting (smaller $r$ weight). In simulation, we observed that under the aggressive setting, IBVS-MPC and ESKF-VMPC in the wind case cannot maintain stability and crash at the beginning of the task. To ensure a fair comparison, we therefore report their best performance under the conservative setting, which allows them to complete the mission. By contrast, the proposed ESKF-PRE-VMPC remains stable under the aggressive setting in all cases and achieves higher control accuracy (it also outperforms the baselines under the conservative setting in additional tests). The reason is that the higher effective update rate of the image features enables the MPC to react promptly to system variations, rather than producing control actions based on delayed visual information, which would otherwise degrade responsiveness and may destabilize the underactuated system.
\begin{table}[htbp!]
\vspace{-0.5cm}
\centering
\caption{MPC image feature weight settings for simulations.}
\label{tab:q_settings}
\setlength{\tabcolsep}{3pt} 
\renewcommand{\arraystretch}{1.25}
\begin{tabular}{lcc}
\toprule
\textbf{Setting} & \textbf{$(\phi_\theta,\phi_r)^\mathsf{T}$} & \textbf{Applied in} \\
\midrule
\vspace{0.5em}
\textbf{Aggressive}   & $(0.2,\,1.0)^\mathsf{T}$  &
\makecell[c]{ESKF-VMPC (wind-free)\\ESKF-PRE-VMPC (wind-free \& wind)} \\
\textbf{Conservative} & $(0.3,\,0.3)^\mathsf{T}$ &
\makecell[c]{IBVS-MPC (wind-free \& wind)\\ESKF-VMPC (wind)} \\
\bottomrule
\end{tabular}
\vspace{-0.5cm}
\end{table}

To practically validate the proposed framework, real-world experiments were further conducted on a nano quadrotor platform Crazyflie~\cite{giernacki2017crazyflie}. Due to the limited payload capacity of the original Crazyflie platform, which is barely sufficient to carry components beyond motion capture markers, the propulsion system is modified to increase its payload capability to at least 6 g. This enhancement enables the onboard integration of a lightweight camera and image-transmission module (as shown in Fig.~\ref{fig:crazyflie}), allowing fast indoor algorithm validation. 
By improving the payload capacity of the Crazyflie, a wider range of algorithms can be experimentally validated in space-constrained indoor environments. Detailed modification procedures are available at \url{https://github.com/lw-seu/Crazyflie-modification}.

\subsection{Real-world Validation}
\subsubsection{Real-world setup}
\begin{figure}[htbp!]
    \centering
    \includegraphics[width=0.48\linewidth]{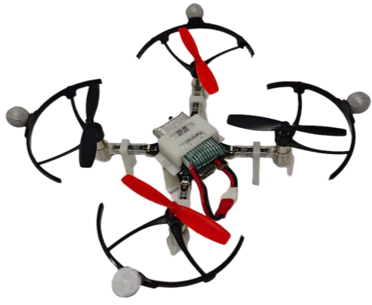}
    \includegraphics[width=0.42\linewidth]{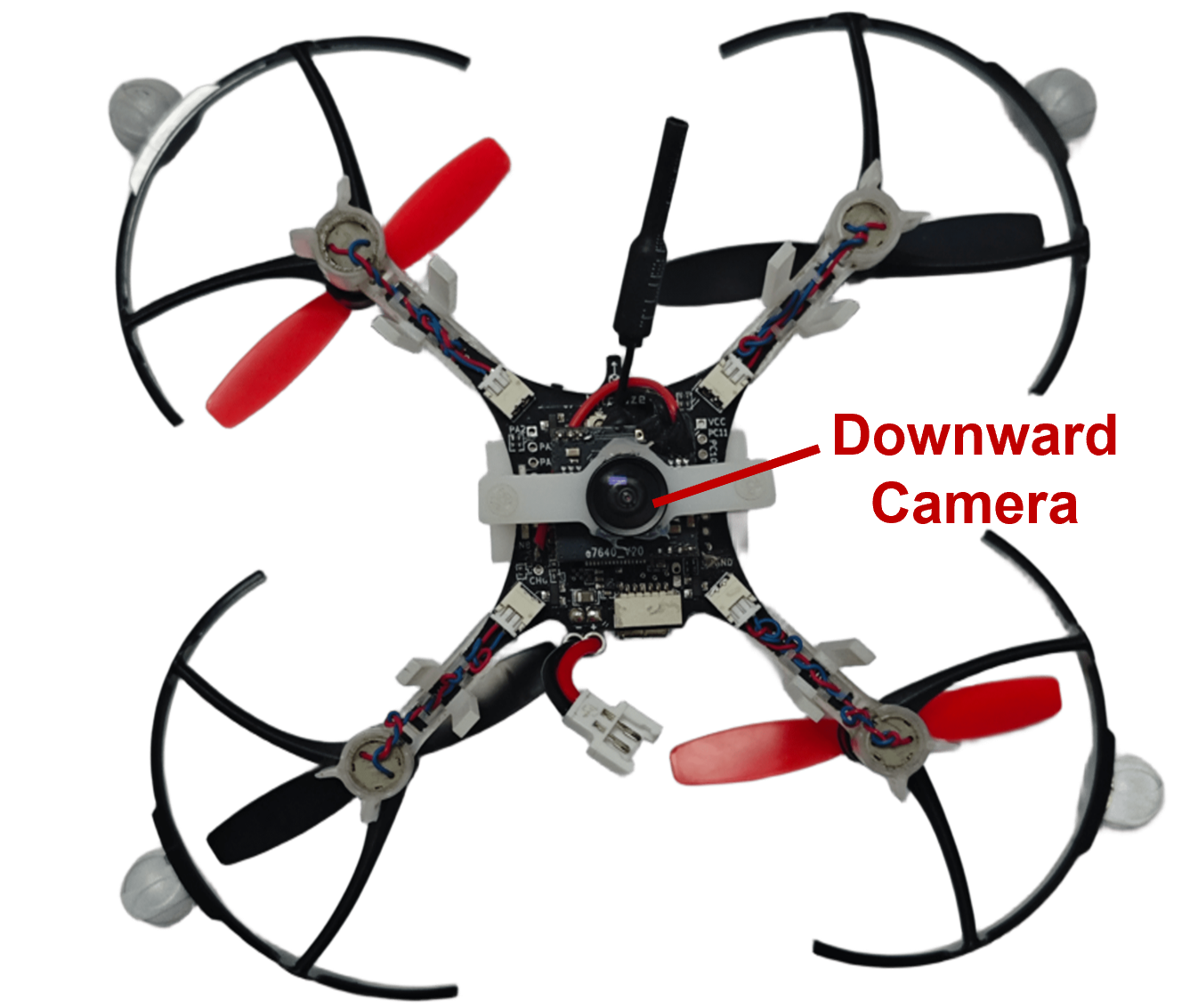}
    \caption{\textbf{Modified Crazyflie platform} with a downward camera with a resolution of 756 $\times$ 560. Left: side view; Right: bottom view.}
    \label{fig:crazyflie}
    \vspace{-0.1cm}
\end{figure}

\begin{figure}[htbp!]
    \centering
    \includegraphics[width=0.9\linewidth]{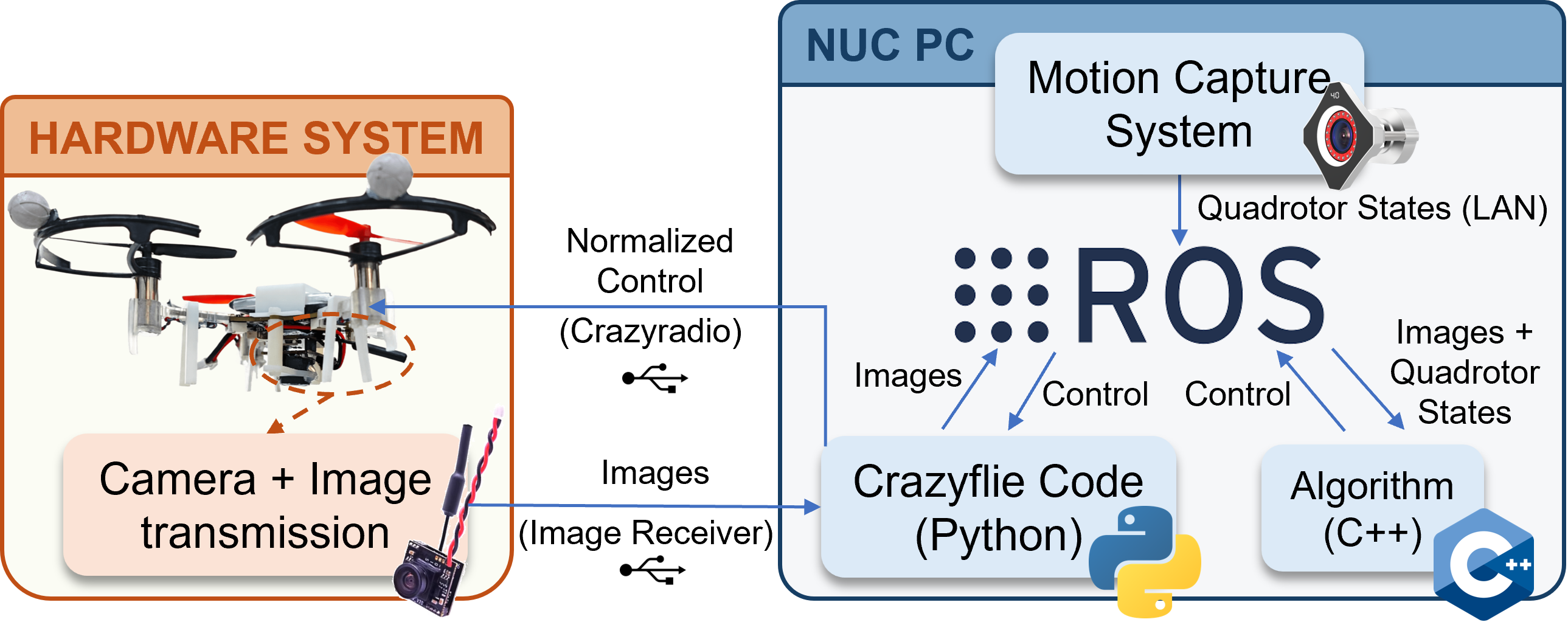}
    \caption{\textbf{Hardware and communication architecture} of the modified Crazyflie platform for real-world experiments.}
    \label{fig:communication}
    \vspace{-0.1cm}
\end{figure}
\begin{figure}[htbp!]
    \centering
    \includegraphics[width=0.75\linewidth]{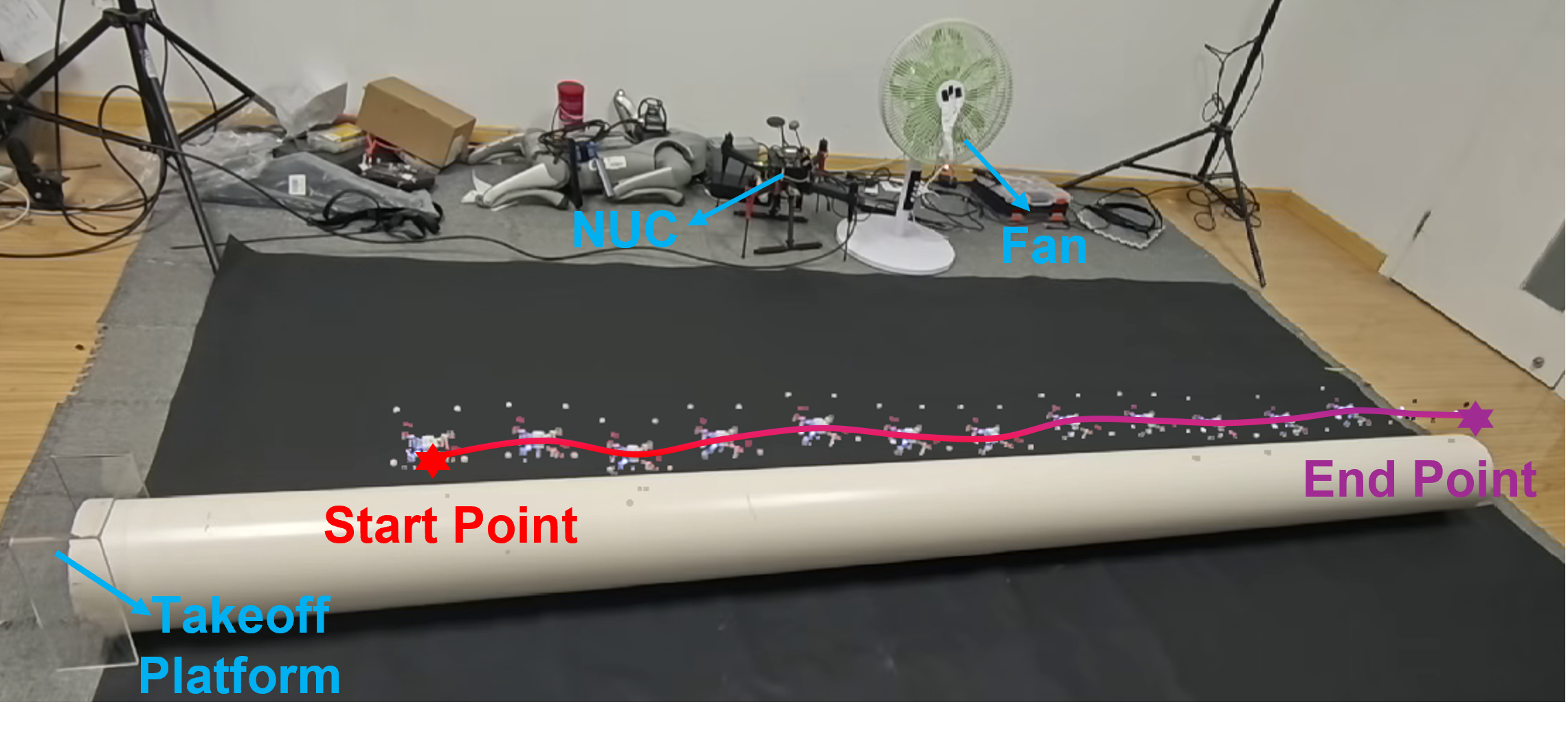}
    \includegraphics[width=0.75\linewidth, height=0.55\linewidth]{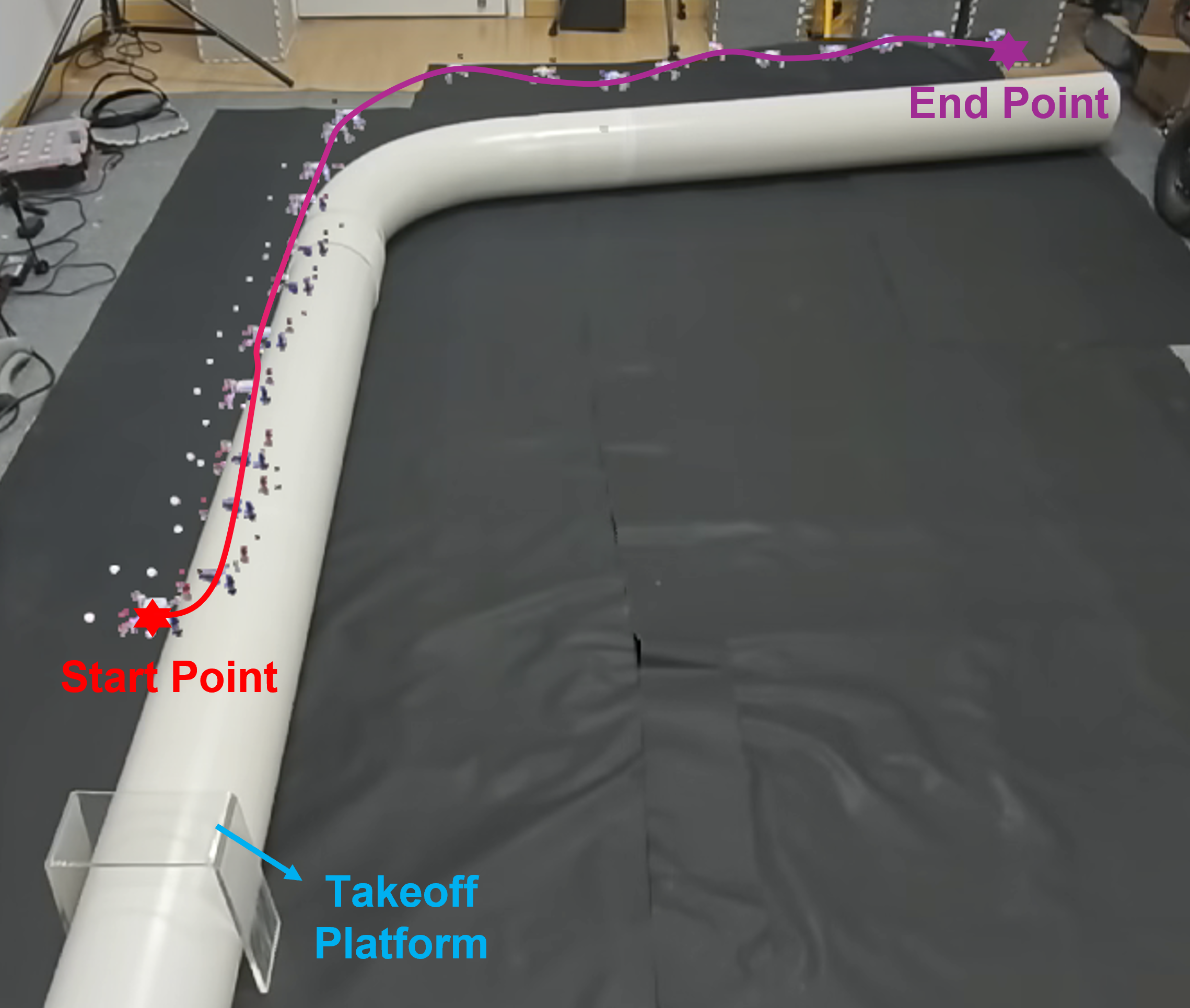}
    \caption{\textbf{Experimental Results}. Top: straight pipeline experiments with an effective inspection length of approximately 2.5 m; Bottom: curved pipeline experiment with an effective inspection length of approximately 4 m.}
    \label{fig:experiment}
    \vspace{-0.1cm}
\end{figure}

To emulate realistic onboard computation, a NUC PC with an Intel i5-12450H CPU is used for algorithm execution. Quadrotor state data are obtained from a four-camera motion capture system (ChingMu MC1300). To enable a modular design, all information on the PC is exchanged via ROS 1, allowing the algorithm to be deployed on different platforms with little or no modification, as shown in Fig.~\ref{fig:communication}. This improves system flexibility and ease of extension, while reducing the effort required for platform migration.

In real-world experiments, we compare ESKF-VMPC and ESKF-PRE-VMPC, as 
both IBVS and IBVS-MPC fail in the wind-disturbance simulation scenario. Three sets of comparative experiments are conducted: (1) straight pipeline without wind disturbance (\textbf{SP-NW}); (2) straight pipeline with wind disturbance (\textbf{SP-W}); and (3) pipeline with a near 100$^{\circ}$ bend segment inserted (\textbf{BP}). The experimental environment is set up as shown in Fig.~\ref{fig:experiment}. 

Since computer vision is not the primary focus of this work, a black cloth is placed beneath the pipeline, and cylindrical PVC pipes are used to reduce the complexity of visual processing. In addition, due to the limited capability of the onboard micro camera and the use of analog image transmission, the visual input is often degraded by poor reception and severe noise. As shown in Fig.~\ref{fig:failure vision}, the proposed method remains effective even when visual measurements are unreliable. The incorporation of image feature prediction enables the system to maintain consistent feature estimation, reducing sensitivity to degraded or misleading observations and improving overall robustness.
\begin{figure}[htbp!]
    \centering
    \includegraphics[width=0.4\linewidth]{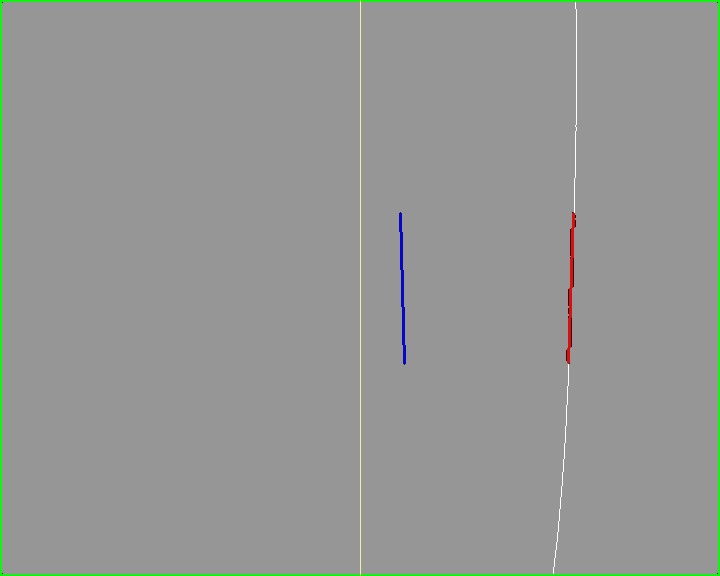}
    \includegraphics[width=0.4\linewidth]{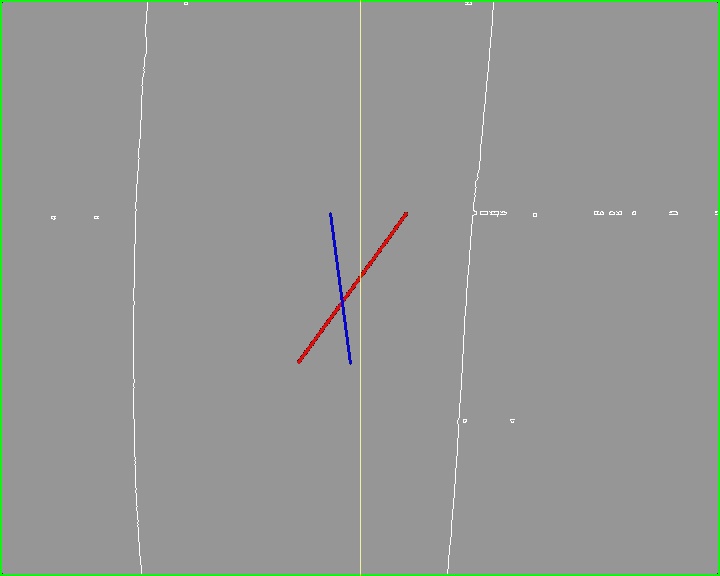}
    \caption{\textbf{Image feature estimation} in the presence of visual detection failures. \cbox[gray!50]{\textcolor{white}{White lines}} denote the detected edges; \textcolor{red}{red lines} denotes the detected midline; \textcolor{blue}{blue lines} denotes the predicted midline; \textcolor{myyellow}{yellow lines} denotes the image centerline. Left: incomplete pipeline edge detection; Right: erroneous detection caused by severe image noise. When visual measurements are unavailable (red lines), the estimation of image features (blue lines) provides reliable feature information.}
    \label{fig:failure vision}
    \vspace{-0.1cm}
\end{figure}

Furthermore, although wind disturbances are generated by a fan, the confined indoor environment leads to strong airflow affecting the entire experimental area. These challenging conditions impose strict requirements on the stability and robustness of the comparative algorithms.

\subsubsection{Real-world results}
\begin{table}[htbp!]
\centering
\caption{\small RMSE of $\theta$ and $r$ for ESKF-VMPC and ESKF-PRE-VMPC (lower is better). Best in \textcolor{blue}{blue bold}, failure by ``--''}
\label{exp_comparison}
\begin{threeparttable}
\setlength{\tabcolsep}{3pt}
\renewcommand{\arraystretch}{1.2}
\begin{tabular}{c|cc|cc|cc}
\Xcline{1-7}{1pt}
\multirow{3}{*}{\textbf{Methods}} 
& \multicolumn{6}{c}{\textbf{RMSE} $\times 10^{-2}$} \\ \Xcline{2-7}{0.4pt}
& \multicolumn{2}{c|}{\textbf{SP-NW}} & \multicolumn{2}{c|}{\textbf{SP-W}} & \multicolumn{2}{c}{\textbf{BP}} \\ \Xcline{2-7}{0.4pt}
& $\theta$ & $r$ & $\theta$ & $r$ & $\theta$ & $r$  \\
\Xcline{1-7}{0.4pt}
ESKF-VMPC & 4.37 & 6.61  & --&--&--&-- \\
ESKF-PRE-VMPC & \textcolor{blue}{\textbf{2.07}} & \textcolor{blue}{\textbf{1.65}} & \textcolor{blue}{\textbf{7.23}}&\textcolor{blue}{\textbf{7.57}}&\textcolor{blue}{\textbf{8.18}}&\textcolor{blue}{\textbf{11.65}} \\
\Xcline{1-7}{1pt}
\end{tabular}
\end{threeparttable}
\end{table}

\begin{figure}[htbp!]
\centering
\subfloat[\textbf{SP-NW} (straight pipeline without wind
disturbance ).]{%
    \centering
    \includegraphics[width=0.95\linewidth]{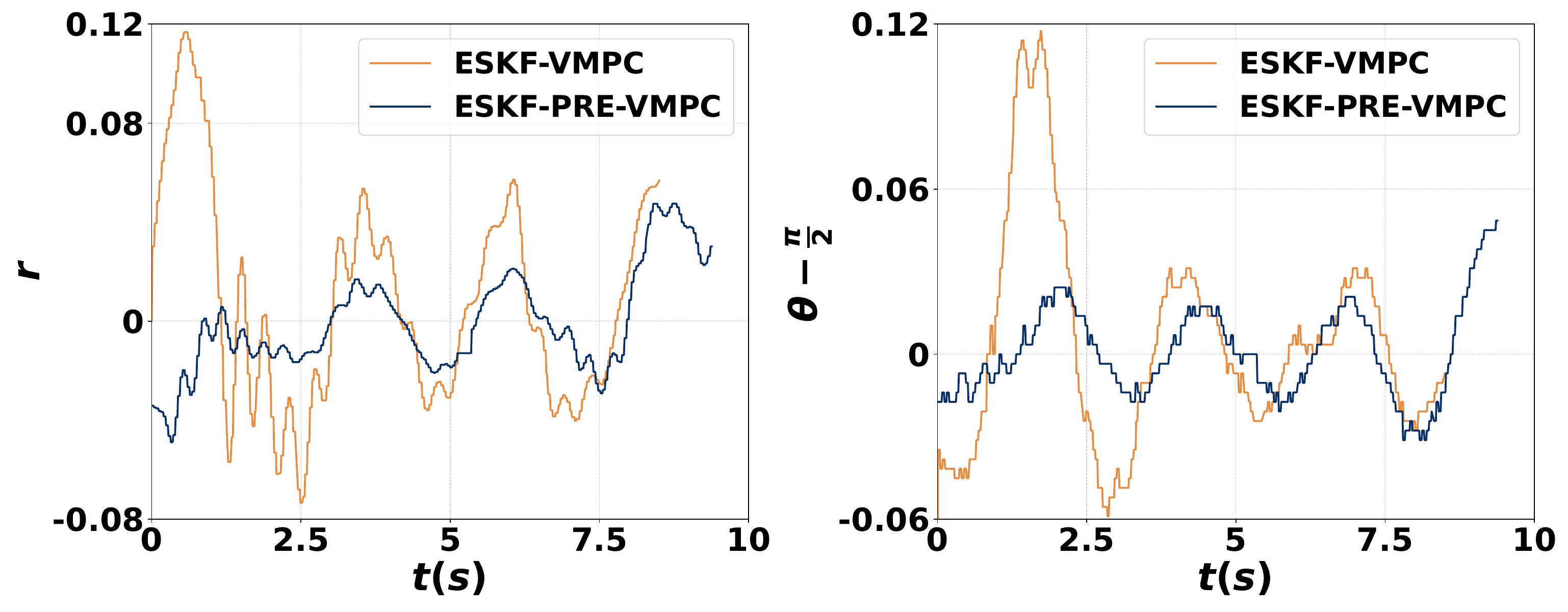}%
    \label{fig:exp_eval_a}%
}
\vspace{0.1em}
\subfloat[\textbf{SP-W} (straight pipeline with wind
disturbance)]{%
    \centering
    \includegraphics[width=0.95\linewidth]{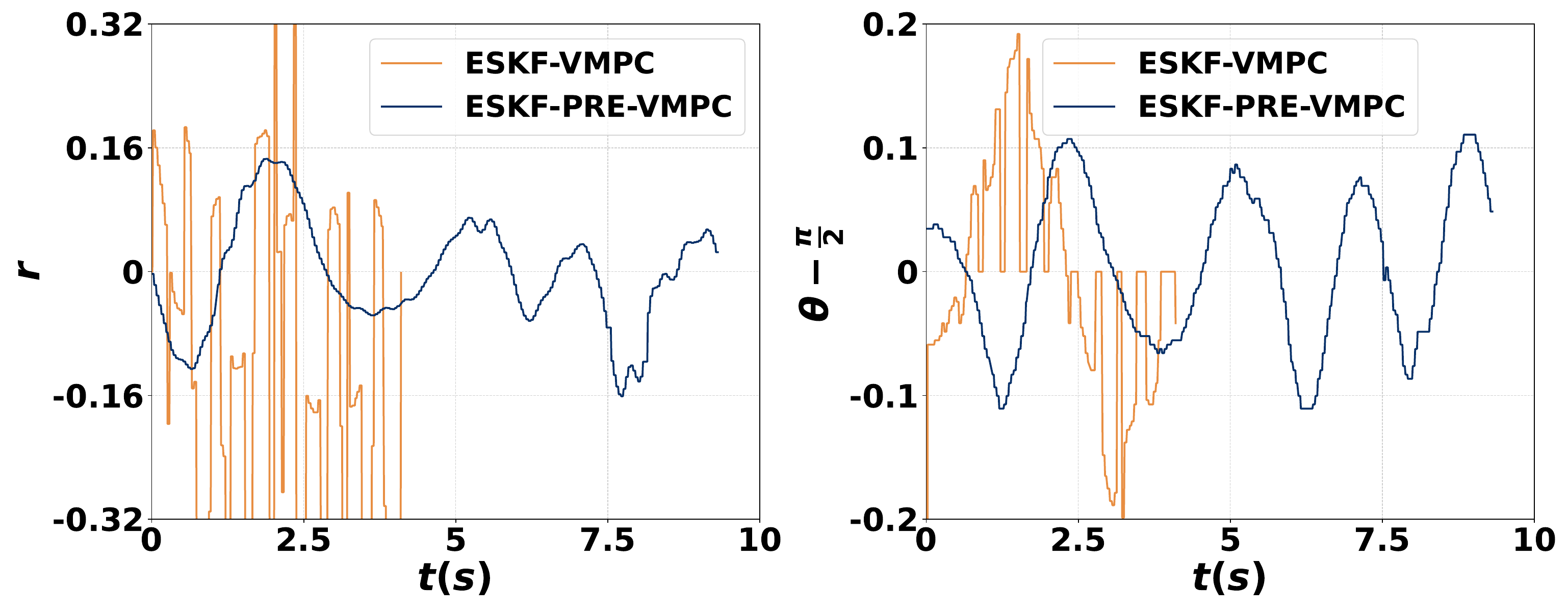}%
    \label{fig:esxp_eval_b}%
}
\caption{\textbf{Comparative results} of ESKF-VMPC and ESKF-PRE-VMPC.}
\label{fig:exp_evaluation}
\vspace{-0.1cm}
\end{figure}

As shown in Table~\ref{exp_comparison} and Fig.~\ref{fig:exp_evaluation}, the proposed ESKF-PRE-VMPC consistently achieves superior performance in all scenarios. In the SP-NW case, both methods successfully complete the inspection task; however, ESKF-PRE-VMPC achieves significantly lower RMSE values, reducing the error by \(52.63\%\) in \(\theta\) and \(75.04\%\) in \(r\). In the SP-W and BP cases, the performance difference becomes more pronounced. The baseline ESKF-VMPC fails to maintain stable operation and eventually crashes, whereas the proposed ESKF-PRE-VMPC still successfully completes the inspection task. These results demonstrate that the proposed method effectively enhances robustness under external disturbances and challenging scenarios involving large-angle turns.


\section{CONCLUSIONS}

This paper presented an autonomous quadrotor pipeline inspection framework for 3D environments based on visual servoing model predictive control. By coupling quadrotor dynamics with image feature kinematics, a unified predictive model was established for control-oriented visual servoing. To handle low-rate visual updates, measurement degradation, and environmental uncertainties, an extended-state Kalman filtering scheme with image feature prediction was integrated into VMPC, yielding the ESKF-PRE-VMPC framework. A terrain-adaptive velocity design was also introduced to maintain inspection speed while adjusting vertical motion to terrain variations without prior terrain knowledge. Gazebo simulations and real-world experiments showed more stable feature regulation, stronger robustness, and more reliable tracking than representative baselines.

These results indicate that combining image feature prediction with disturbance-aware visual servoing MPC effectively mitigates the mismatch between low-rate visual sensing and high-rate control in near-proximity UAV inspection. Future work will validate the framework in more realistic industrial settings with stronger disturbances, harsher perception conditions, more complex pipeline geometries, and tighter integration of visual detection and uncertainty-aware prediction.




\bibliographystyle{IEEEtran}
\bibliography{ref}

\end{document}